\documentclass[12pt]{article}

\usepackage[left=1in,top=1in,right=1in,bottom=1in,head=.1in,nofoot]{geometry}
\usepackage{setspace,url,bm,amsmath} 
\usepackage{titlesec} 
\titlelabel{\thetitle.\quad} 
\usepackage[margin=20pt]{subcaption}
\usepackage{amsmath,amsfonts,amsthm,amssymb}
\usepackage{graphicx}
\usepackage[colorlinks=true,linkcolor=black,citecolor=blue,urlcolor=blue]{hyperref}
\usepackage[table]{xcolor}
\usepackage{multirow}
\usepackage{array}
\usepackage{algorithm}%
\usepackage{algorithmicx}%
\usepackage{algpseudocode}%
\usepackage{booktabs}
\usepackage{threeparttable}
\usepackage{comment}
\usepackage{float}
\usepackage{natbib}
\usepackage{indentfirst}
\usepackage{mathrsfs}

\usepackage{changepage} 
\usepackage{siunitx}    
\usepackage{caption}    

\usepackage{appendix}

\usepackage{longtable}
\usepackage{amsfonts}

\newtheorem*{theorem*}{Theorem}
\newtheorem{theorem}{Theorem}
\newtheorem{proposition}{Proposition}
\newtheorem{lemma}{Lemma}

\newtheorem{example}{Example}

\newtheorem{corollary}{Corollary}
\newtheorem*{corollary*}{Corollary}

\newcommand{\RNum}[1]{\uppercase\expandafter{\romannumeral #1\relax}}

\newcommand \ind[1] {I(#1)}

\newcommand \Prob[1]{\mathbb{P}(#1)}
\newcommand \ProbB[1]{\mathbb{P}\{#1\}}

\newcommand \inp {\overset{P}{\rightarrow}}

\newcommand \inas {\overset{\text{a.s.}}{\longrightarrow}}

\newcommand \vecX{\textbf{x}}
\newcommand\Exy{\mathbb{E}_{(\textbf{x}, y)}}
\newcommand\Ex{\mathbb{E}_{\textbf{x}}}
\newcommand\Pxy{\mathbb{P}_{(\mathbf{x}, y)}}

\begin{document}
\begin{singlespace}
\title{\bf Covariance-Driven Regression Trees: Reducing Overfitting in CART}

\author{
\small
{
Likun Zhang and Wei Ma\thanks{ Likun Zhang is PhD student, Institute of Statistics and Big Data, Renmin University of China (E-mail: zhanglk6@ruc.edu.cn). Wei Ma is Associate Professor, Institute of Statistics and Big Data, Renmin University of China (E-mail: mawei@ruc.edu.cn).}
}
\\ \\
{\small Institute of Statistics and Big Data, Renmin University of China, Beijing, China}
}

\date{}
\maketitle
\end{singlespace}

\thispagestyle{empty}
\vskip -8mm 

\begin{abstract}
Decision trees are powerful machine learning algorithms, widely used in fields such as economics and medicine for their simplicity and interpretability. However, decision trees such as CART are prone to overfitting, especially when grown deep or the sample size is small. Conventional methods to reduce overfitting include pre-pruning and post-pruning, which constrain the growth of uninformative branches. In this paper, we propose a complementary approach by introducing a covariance-driven splitting criterion for regression trees (CovRT). This method is more robust to overfitting than the empirical risk minimization criterion used in CART, as it produces more balanced and stable splits and more effectively identifies covariates with true signals. We establish an oracle inequality of CovRT and prove that its predictive accuracy is comparable to that of CART in high-dimensional settings. We find that CovRT achieves superior prediction accuracy compared to CART in both simulations and real-world tasks. 
\end{abstract}

\vspace{12pt}


\newpage

\clearpage
\setcounter{page}{1}

\allowdisplaybreaks
\baselineskip=24pt


\section{Introduction}
Decision trees are among the most popular methods in statistical learning, owing to their high interpretability and strong predictive performance across a wide range of tasks \citep{podgorelec2002decision, charbuty2021classification}. Decision trees recursively partition the covariate space until a stopping rule is met, resulting in a hierarchical, tree-like predictive model. Each internal node represents a split based on the values of covariates determined by a specific splitting criterion, while each terminal node, or leaf, corresponds to a final decision within a partitioned region of the covariate space. Given the computational infeasibility of finding the globally optimal partition, the seminal Classification and Regression Trees (CART) algorithm \citep{breiman1984} employs a coordinate-wise, top-down greedy strategy that recursively performs binary splits to maximize the reduction in impurity at each step. CART is arguably the most widely used variant of decision trees and serves as a building block for ensemble methods such as random forests \citep{breiman2001random}.

In regression analysis, where the response variable is real-valued, CART uses the average of squared residuals (i.e., empirical $L_2$ risk) as the impurity measure and constructs a piecewise-constant predictive model, with each leaf node fitted using the sample mean of the training observations within that region. At each internal node, CART performs an exhaustive search over covariates and split points to identify the split that minimizes the empirical $L_2$ risk. This splitting strategy is asymptotically optimal in the sense that, when the maximum tree depth is fixed and the training sample size grows to infinity, minimizing the empirical $L_2$ risk leads to a predictive model that asymptotically achieves the minimum generalization error. However, real-world datasets are of finite size, and as the tree grows deeper, fewer observations are available within each node. This can lead to unstable splits and cause decision trees to overfit the training data, resulting in poor generalization to unseen data \citep{Bramer2007}. Common options for reducing overfitting include pre-pruning, which prevents the growth of uninformative branches, and post-pruning, which removes such branches after the tree is fully grown. In this paper, we propose a complementary approach by introducing a covariance-driven splitting criterion for regression trees (CovRT) that exhibits greater robustness to overfitting than the standard CART procedure. In the high-dimensional setting, we theoretically establish that CovRT achieves predictive accuracy at least comparable to that of CART. Moreover, we show that CovRT produces more balanced and stable splits and is more likely to identify covariates with true signals. We further demonstrate its superior performance through simulations and real-world applications.

Although CART has been proposed and widely used in practice for over 40 years, its mathematical properties remain largely unknown. \citet{breiman1984} provided the first consistency result for CART at the time of its introduction, under a stringent minimum node size condition. However, this condition is impractical, especially for CART-based random forests, where individual trees are typically fully grown. As Breiman’s original algorithm is complex to analyze directly, many works have established the consistency of individual trees and ensemble forests under simplified splitting rules that depend less on the response variable, often by introducing additional randomness into the splitting procedure \citep{biau2008, biau2012, denil2014narrowing, gao2022towards, cai2023extrapolated}. However, such simplifications undermine the ability of decision trees to effectively handle high-dimensional data \citep{tan2024statistical}. \citet{scornet2015} first proved the consistency of the original CART algorithm for additive models, without requiring strong assumptions on the tree growth procedure. In a high-dimensional setting where the number of covariates can grow sub-exponentially with the sample size, \cite{klusowski2024} established the consistency of CART for additive models under $\ell_0$ or $\ell_1$ sparsity constraints.  Building on this foundation, our work makes a distinct theoretical advancement by establishing the consistency of CovRT within the same high-dimensional framework. Notably, we derive an oracle inequality for CovRT that matches the form of the oracle inequality established for CART in \cite{klusowski2024}. This result highlights the strength of CovRT in high-dimensional problems and underscores its potential as a robust alternative to empirical risk minimization strategies. Furthermore, through simulations and empirical studies, we show that CovRT has smaller generalization error than CART. Moreover, its node-splitting criterion is covariance-driven, aiming to identify the covariate and the split point that most strongly influence the regression function, which provides a statistically natural interpretation for each split.

	\section{Framework and Notation}\label{section:notation}
In this paper, we focus on the nonparametric regression analysis with a random vector (\vecX, $y$), where $\vecX \in \mathbb{R}^p$ is the covariate vector and $y \in \mathbb{R}$ is the response variable. Suppose that the training data $\mathcal{D}:=\left\{\left(\mathbf{x}_1, y_1\right),\left(\mathbf{x}_2, y_2\right), \ldots,\left(\mathbf{x}_N, y_N\right)\right\}$ is sampled independently and identically distributed (i.i.d.) from the joint distribution $\Pxy=\mathbb{P}_{\mathbf{x}} \mathbb{P}_{y \mid \mathbf{x}}$. We want to study the relation between \vecX \ and $y$ by identifying a function $g: \mathbb{R}^p \rightarrow \mathbb{R}$ such that $g(\vecX)$ predicts $y$ effectively. Specifically, denote the population and empirical $L_2$ risk (mean squared error) of a function $g$ by 
\[
	\mathcal{R}(g):=\Exy \{y- g(\mathbf{x})\}^2, \quad \widehat{\mathcal{R}}(g):=\frac{1}{N} \sum_{i=1}^N \left\{y_i - g\left(\mathbf{x}_i\right)\right\}^2,
    \]
respectively. Our objective is to find a function $g$ that minimizes the population $L_2$ risk $\mathcal{R}(g)$. Minimizing the population risk ensures that the trained model generalizes well to unseen data, rather than merely fitting the training samples. It is well known that the regression function 
\[
g^*(\vecX):= \Exy(y\mid \vecX)
\]
minimizes $\mathcal{R}(g)
$. In applications it is impossible to predict $y$ using $g^*(\vecX)$ because the distribution $\Pxy$ is unknown. Therefore, we aim to use the training data to construct a $\mathcal{D}$-dependent function $g$ that approximates $g^*$. We evaluate the generalization performance of a predictive model 
$g$ using the squared $L_2$ distance between $g$ and the true regression function $g^*$, the $L_2$ error
\begin{equation}
\left\|g-g^*\right\|^2:=\Ex\left\{g(\mathbf{x})-g^*(\mathbf{x})\right\}^2 = \mathcal{R}(g)-\mathcal{R}\left(g^*\right).
\label{eq:norm}
\end{equation}

To investigate the high-dimensional properties of our Covariance-Driven Regression Trees (CovRT) proposed in the next section, we start by introducing the additive modeling framework as a basis for our analysis. Consider the additive function class
\[
\mathcal{G}^1:=\left\{g(\vecX):=g_1\left(x_1\right)+g_2\left(x_2\right)+\cdots+g_p\left(x_p\right)\right\},
\]
where each $g_j\left(x_j\right)$, for $j = 1, \ldots, p$, is a univariate measurable function defined on the interval $[a_j, b_j]$. In the additive modeling framework, our goal is to find a function $g(\vecX) \in \mathcal{G}^1$ that approximates the regression function $g^*(\vecX)$, although $g^*(\vecX)$ may or may not belong to $\mathcal{G}^1$.

Additive models \citep{hastie1986}, which generalize linear models, combine the flexibility of nonparametric methods with the ease of interpretability and estimation. They play an important role in modeling complex high-dimensional data and are commonly used in the theoretical analysis of decision trees. Assuming that the regression function follows a fixed-dimensional additive model with continuous component functions, \citet{scornet2015} established the consistency of CART and Breiman’s random forests. Building on a more general additive modeling framework, \citet{klusowski2024} developed a unified theory that proves high-dimensional consistency for both the CART and C4.5 algorithms. See also \cite{klusowski2020sparse} and \cite{klusowski2021nonparametric} for more analyses of tree-based methods under additive models, and \citet{tan2022cautionary} for a cautionary discussion of the limitations of using such methods for detecting global structure.

For $g(\vecX) \in \mathcal{G}^1$, we define the \textit{total variation $\ell_1$ norm} $\Vert g\Vert_{\text{TV}}$ to characterize sparsity. Our theoretical framework does not require the covariates to be mutually independent. To avoid identifiability issues, we therefore define $\|g\|_{\mathrm{TV}}$ as the infimum of
\begin{equation}
	\text{TV}(g_1) + \text{TV}(g_2) + \cdots + \text{TV}(g_p)
    \label{eq:TV_norm}
\end{equation}
 taken over all additive decompositions $g(\vecX) = g_1\left(x_1\right)+g_2\left(x_2\right)+\cdots+g_p\left(x_p\right)$, where $\text{TV}(g_j)$ denotes the total variation of the univariate function $g_j$, defined as 
 \[
    \text{TV}(g_j) = \sup \sum_{i=1}^{m-1}\mid g_j(z_{i+1}) - g_j(z_{i})\mid,
 \]
 where the supremum is taken over all finite partitions $a_j \leq z_1 < z_2 < \cdots < z_m \leq b_j$ of the domain of $g_j$. In other words, $\Vert g\Vert_{\text{TV}}$ represents the aggregated total variation of the
 individual component functions \citep{tan2019doubly}.  In high-dimensional settings, it is also common to consider regression functions that are sparse in the sense that they depend meaningfully on only a small subset of the covariates. This sparsity can be quantified by the so-called  $\textit{$\ell_0$ norm}$, denoted by $\Vert g\Vert_{\ell_0}$, which
 is defined as the infimum of 
 \[
 \#\{g:g_j(\cdot) \text{ is non-constant}\}
 \]
 over all additive decompositions of $g(\cdot).$ Note that 
 \[
 \Vert g\Vert_{\text{TV}} \leq \Vert g\Vert_{\ell_0}\cdot \max_{j}\text{TV}(g_j).
 \]

\section{Regression Tree Construction}
\subsection{CovRT Methodology}
Consider a regression tree $T$, and let 
t be a node in the tree corresponding to a rectangular subregion of the covariate space. Let $s$ be a candidate split point for covariate $x_j$ that partitions the parent node t into two daughter nodes: the left node $\mathrm{t}_L := \{\vecX \in \mathrm{t} : x_j \leq s\}$ and the right node $\mathrm{t}_R := \{\vecX \in \mathrm{t} : x_j > s\}$. The left and right daughter nodes $\mathrm{t}_L$ and $\mathrm{t}_R$ then become new parent nodes and are recursively split in subsequent steps. Nodes where splitting stops are called terminal nodes or leaves. To make a prediction at a point $\vecX$ using a tree $T$, we assign it the sample mean of the response values $y_i$ within the terminal node t that contains $\vecX$:
\begin{equation}
\hat{g}(T)(\mathbf{x}):=\bar{y}_{\mathrm{t}}, \quad \mathbf{x} \in \mathrm{t}.
\label{eq:prediction}
\end{equation} We adopt the same coordinate-wise recursive binary splitting strategy as in CART to search for the optimal partition. However, instead of splitting a node based on empirical $L_2$ risk as in CART, we first propose a new population-level target parameter for the split, which has a special interpretation within the additive function class $\mathcal{G}^1$. At each node t, we consider the optimal split $(j^*, s^*)$ that maximizes covariance-squared
\begin{equation}
\begin{aligned}
\mathcal{CS}(j, s, \mathrm{t})&:= \operatorname{Cov}^2(\ind{x_j\leq s}, y \mid \vecX \in \mathrm{t})\\
&= P_{\mathrm{t}_L}^2P_{\mathrm{t}_R}^2\left\{\mathbb{E}(y\mid \vecX \in \mathrm{t}_L) - \mathbb{E}(y\mid \vecX \in \mathrm{t}_R)\right\}^2 
 , \quad\mathcal{CS}(\mathrm{t}):=\max _{(j, s)} \mathcal{CS}(j, s, \mathrm{t}),
\label{eq:heter_gain}
\end{aligned}
\end{equation}
where $P_{\mathrm{t}_L}:=\Prob{\vecX \in \mathrm{t}_L\mid\vecX \in \mathrm{t}}$ and $P_{\mathrm{t}_R}:=\Prob{\vecX \in \mathrm{t}_R\mid\vecX \in \mathrm{t}}$.
The following Theorem \ref{thm:HT} shows that the resulting split $(j^*, s^*)$ aims to identify the covariate and the split point that have the strongest influence on the regression function $g^*(\mathbf{x})$.

\begin{theorem}
Suppose that $g^*(\mathbf{x}) \in \mathcal{G}^1$ and that the coordinates $x_1, \ldots, x_p$ are mutually independent, each with marginal distribution $F_{\vecX_j}$. Let $F_{\vecX_j\mid \mathrm{t}}$ denote the distribution of $x_j$ conditional on $\mathbf{x} \in \mathrm{t}$. Then at any node $\mathrm{t} = (a_1, b_1] \times (a_2, b_2] \times \cdots \times (a_p, b_p]$, and for the splitting criterion $\mathcal{CS}(j, s, \mathrm{t})$ defined in (\ref{eq:heter_gain}), we have

\begin{equation}
\begin{aligned}
\mathcal{CS}(j, s, \mathrm{t})
&=\left[ {\int_{a_j}^{s}\{\mathbb{E}(y \mid x_j = x, \mathbf{x} \in \mathrm{t}) - \mathbb{E}(y\mid \mathbf{x} \in \mathrm{t})\}dF_{\vecX_j\mid \mathrm{t}}(x)} \right]^2 \\
&=\frac{1}4\Bigg[ {\int_{a_j}^{s}\{\mathbb{E}(y \mid x_j = x, \mathbf{x} \in \mathrm{t}) - \mathbb{E}(y\mid \mathbf{x} \in \mathrm{t})\}dF_{\vecX_j\mid \mathrm{t}}(x)} \\
&\hspace{125pt}-{\int_{s}^{b_j}\{\mathbb{E}(y \mid x_j = x, \mathbf{x} \in \mathrm{t}) - \mathbb{E}(y\mid \mathbf{x} \in \mathrm{t})\}dF_{\vecX_j\mid \mathrm{t}}(x)} \Bigg]^2. \\
\end{aligned}
\label{eq:HT}
\end{equation}
\label{thm:HT}
\end{theorem}
We use the left-open, right-closed interval format purely for notational convenience and follow it throughout the paper. The second equality in (\ref{eq:HT}) follows from the fact that
\[
\int_{a_j}^{b_j}
\left\{
\mathbb{E}(y \mid x_j = x, \mathbf{x} \in \mathrm{t})
-
\mathbb{E}(y \mid \mathbf{x} \in \mathrm{t})
\right\}
\, dF_{x_j \mid \mathrm{t}}(x)
= 0.
\]
Theorem \ref{thm:HT} says at any node t, the splitting criterion $\mathcal{CS}(j, s, \mathrm{t})$ measures the cumulative deviation of the conditional expectation $\mathbb{E}(g^*(\mathbf{x}) \mid x_j = x, \mathbf{x} \in \mathrm{t})$ from $\mathbb{E}(g^*(\mathbf{x})\mid \mathbf{x} \in \mathrm{t})$ in the direction of $x_j$. The optimal split $(j^*, s^*)$ in each step chooses the covariate $x_j$ on which the conditional expectation has the maximum deviation and the split point $s$ that the deviation can be best separated in the two daughter nodes $\mathrm{t}_L = \{\vecX \in \mathrm{t} : x_j \leq s\}$ and $\mathrm{t}_R = \{\vecX \in \mathrm{t} : x_j > s\}$. Some examples are given below to illustrate our covariance-squared splitting criterion:
 
\begin{example}[Multivariate linear model]
     Let $g^*(\mathbf{x}) = \alpha + \sum_{j=1}^p\beta_jx_j$, where $\mathbf{x}$ is uniformly distributed on $(0, 1]^p$. At each node $\mathrm{t} = (a_1, b_1] \times (a_2, b_2] \times \cdots \times (a_p, b_p] \subset (0, 1]^p$,
 \begin{equation}
\begin{aligned}
\max_{s}\mathcal{CS}(j, s, \mathrm{t}) &=  \mathcal{CS}(j, {(a_j+b_j)}/{2}, \mathrm{t}) \\
&= \frac{1}{64}(b_j - a_j)^2\beta_j^2,
\end{aligned}
\end{equation}
and the optimal split is $(j^*, {(a_{j^*}+b_{j^*})}/{2})$, where $j^* = \operatorname{argmax}_j (b_j - a_j)^2\beta_j^2$. So at the root node $(0, 1]^p$, the covariate $x_j$ which has the largest influence on $y$ will be split at first and the consequence split will be sequentially ordered by $(b_j - a_j)^2\beta_j^2$. 
\end{example}

\begin{example}[Monotone component functions]
Suppose that $g^*(\mathbf{x}) = \sum_{j=1}^pg_j\left(x_j\right)$, where $\mathbf{x}$ is uniformly distributed on some rectangle in $\mathbb{R}^p$. For continuous and monotone $g_j$, at each node t, $D(s;j, t):=\mathbb{E}(g^*(\mathbf{x}) \mid x_j = s, \mathbf{x} \in \mathrm{t}) - \mathbb{E}(g^*(\mathbf{x})\mid \mathbf{x} \in \mathrm{t})$ is also monotone in $s$, and hence $\operatorname{argmax}_s \mathcal{CS}(j, s, \mathrm{t})$ is the unique zero point of function $D(s;j, t)$, which means that $\mathbb{E}(g^*(\mathbf{x}) \mid x_j = s, \mathbf{x} \in \mathrm{t}) = \mathbb{E}(g^*(\mathbf{x})\mid \mathbf{x} \in \mathrm{t})$. See Figure \ref{fig:mono_fit_plot} for depth-1 (orange solid line) and depth-2 (green solid line) trees split by $\mathcal{CS}(j, s, \mathrm{t})$ with $g^*(x) = x^3$ (black line) and $x$ uniformly distributed on $(-1, 1]$.

\begin{figure}[H]
    \centering
    \includegraphics[width=0.5\textwidth]{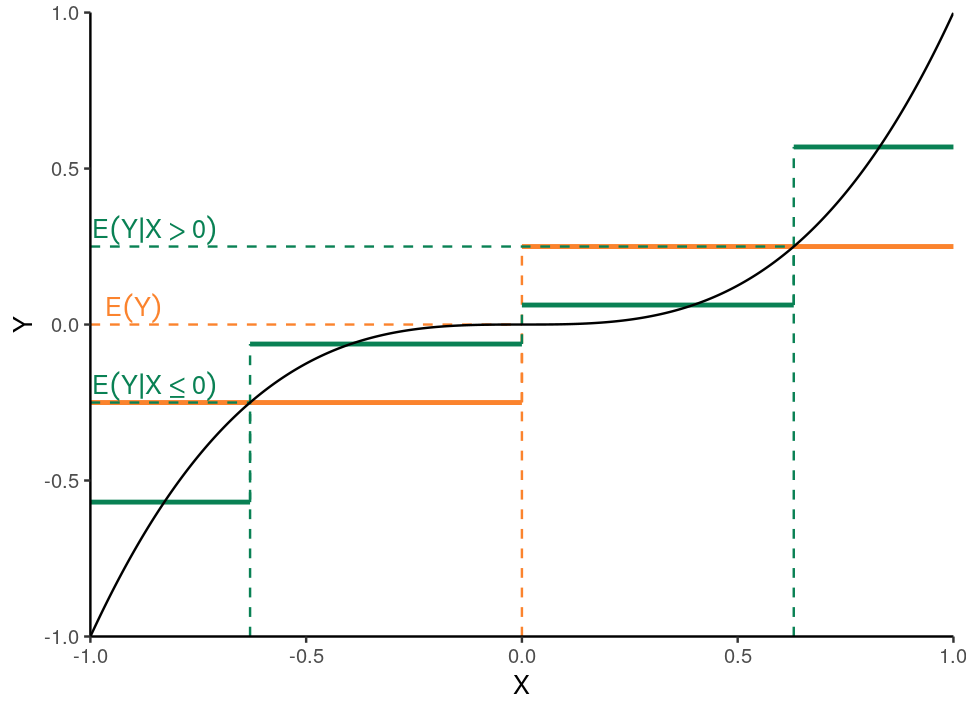}
    \caption{Depth-1 (orange solid line) and depth-2 (green solid line) trees split by $\mathcal{CS}(j, s, \mathrm{t})$ with $g^*(x) = x^3$ (black line) and $x$ uniformly distributed on $(-1, 1]$. The dashed lines show the split points.}
    \label{fig:mono_fit_plot}
\end{figure}
    
\end{example}
In practice, the covariance-squared criterion $\mathcal{CS}(j, s, \mathrm{t})$ defined in (\ref{eq:heter_gain}) is unknown and therefore not directly usable for tree construction. Instead, we use its empirical version as the splitting criterion for CovRT:
\begin{equation}
\begin{aligned}
\widehat{\mathcal{CS}}(j, s, \mathrm{t})
:= \widehat{P}_{\mathrm{t}_L}^2\widehat{P}_{\mathrm{t}_R}^2(\bar{y}_{t_L} - \bar{y}_{t_R})^2 
 , \quad\widehat{\mathcal{CS}}(\mathrm{t}):=\max _{(j, s)} \widehat{\mathcal{CS}}(j, s, \mathrm{t}),
\end{aligned}
\label{eq:heter_gain_feasible}
\end{equation}
where $\widehat{P}_{\mathrm{t}_L}=N_{\mathrm{t}_L} / N_{\mathrm{t}}$ and $\widehat{P}_{\mathrm{t}_R}=N_{\mathrm{t}_R} / N_{\mathrm{t}}$. 

We next provide theoretical justification for the empirical splitting criterion in equation (\ref{eq:heter_gain_feasible}) by showing that it approximates the population criterion in equation (\ref{eq:heter_gain}). For each $j = 1, \ldots, p$, let $ \hat s_j$ denote the split point that maximizes $\widehat{\mathcal{CS}}(j, s, \mathrm{t})$. The following theorem establishes that $\hat s_j$ converges in probability to the corresponding population-optimal split point.
\begin{theorem}
   Suppose that $\mathbb{E}(y^2) < \infty$ and that at each node $\mathrm{t} = (a_1, b_1] \times (a_2, b_2] \times \cdots \times (a_p, b_p]$,  $0 < \Prob{x_j \leq s\mid \mathbf{x} \in \mathrm{t}} < 1$ for all $a_j < s < b_j$ and $j = 1, \ldots,p$. If, for each $j$, the function ${\mathcal{CS}}(j, s, \mathrm{t})$ as a function of $s$ has a unique global maximum $s_j^*$ located in the interior of $(a_j, b_j]$, then as $N \rightarrow \infty$,
   \[
        \hat s_j \inp s_j^* = 
        \underset{a_j < s \leq b_j}{\operatorname{argmax}}{\ \mathcal{CS}}(j, s, \mathrm{t}).
   \]
   \label{thm:empirical_criterion}
\end{theorem}
\subsection{Overfitting in CART}
Denote the within-node population and
empirical $L_2$ risk as

\begin{equation}
	\mathcal{R}_{\mathrm{t}}(g):= \Exy[\left\{y - g\left(\vecX\right)\right\}^2 \mid \vecX \in \mathrm{t}], \quad \widehat{\mathcal{R}}_{\mathrm{t}}(g):=\frac{1}{N_{\mathrm{t}}} \sum_{\mathbf{x}_i \in \mathrm{t}} \left\{y_i - g\left(\mathbf{x}_i\right)\right\}^2, \quad N_t = \#\{\mathbf{x}_i \in \text{t}\},
\end{equation}
respectively. At a node t, CART selects the covariate $x_{j}$ and the split point $s$ that maximize the impurity gain,
\begin{equation}
\mathcal{I} \mathcal{G}(j, s, \mathrm{t}):=\mathcal{I}(\mathrm{t})-\widehat P_{\mathrm{t}_L} \mathcal{I}\left(\mathrm{t}_L\right)-\widehat P_{\mathrm{t}_R} \mathcal{I}\left(\mathrm{t}_R\right), \quad \mathcal{I} \mathcal{G}(\mathrm{t}):=\max _{(j, s)} \mathcal{I} \mathcal{G}(j, s, \mathrm{t}),
\label{eq:impurity_gain}
\end{equation}
where $\widehat P_{\mathrm{t}_L}:=N_{\mathrm{t}_L} / N_{\mathrm{t}}$, $\widehat P_{\mathrm{t}_R}:=N_{\mathrm{t}_R} / N_{\mathrm{t}}$. 
The impurity at a node t, denoted by $\mathcal{I}(\mathrm{t})$, is given by the empirical $L_2$ risk:
\[
\mathcal{I}(\mathrm{t})=\widehat{\mathcal{R}}_{\mathrm{t}}\left(\bar{y}_{\mathrm{t}}\right) = \frac{1}{N_{\mathrm{t}}} \sum_{\mathbf{x}_i \in \mathrm{t}} \left(y_i - \bar{y}_{\mathrm{t}}\right)^2, 
\]
where $\bar{y}_{\mathrm{t}}:=({1}/{N_{\mathrm{t}}}) \sum_{\mathbf{x}_i \in \mathrm{t}} y_i.$ 

When splitting at a node t, maximizing the impurity gain $\mathcal{I} \mathcal{G}(j, s, \mathrm{t})$ defined in (\ref{eq:impurity_gain}) is equivalent to selecting the covariate $x_j$ and split point $s$ such that the resulting daughter nodes $\mathrm{t}_L = \{\vecX \in \mathrm{t} : x_j \leq s\}$ and $\mathrm{t}_R = \{\vecX \in \mathrm{t} : x_j > s\}$ minimize the following expression:
\begin{equation}
\begin{aligned}
\widehat P_{\mathrm{t}_L} \mathcal{I}\left(\mathrm{t}_L\right)+\widehat P_{\mathrm{t}_R} \mathcal{I}\left(\mathrm{t}_R\right) &= \frac{1}{N_{\mathrm{t}}}\left\{\sum_{\mathbf{x}_i \in \mathrm{t}_L} \left(y_i - \bar{y}_{\mathrm{t}_L}\right)^2 + \sum_{\mathbf{x}_i \in \mathrm{t}_R} \left(y_i - \bar{y}_{\mathrm{t}_R}\right)^2\right\} \\
&= \widehat{\mathcal{R}}_{\mathrm{t}}(\bar{y}_{\mathrm{t}_L}\ind{x_j \leq s} + \bar{y}_{\mathrm{t}_R}\ind{x_j > s})\\
\end{aligned}
\label{eq:L2risk_min}
\end{equation}

Under the classical setting where the both the maximum tree depth $K$ and the dimension of covariates $p$ are fixed, CART's empirical $L_2$ risk minimization strategy is  asymptotically optimal for minimizing population $L_2$ risk. In other words, as the sample size $N$ grows infinity, the empirical minimizer $(j, s)$ of equation (\ref{eq:L2risk_min}) also approximately minimizes the population risk $\mathcal{R}_{\mathrm{t}}(\bar{y}_{\mathrm{t}_L}\ind{x_j \leq s} + \bar{y}_{\mathrm{t}_R}\ind{x_j > s})$.
This result holds because when both the covariate dimensionality and tree depth are finite, the class of decision trees has finite Vapnik--Chervonenkis (VC) dimension \citep{athey2021policy}. If we further assume, for example, that the response variable $y$ is bounded, then by the uniform law of large numbers \citep[Chapter~9]{gyorfi2002}, for any node t, as $N \rightarrow \infty$,
\[
\begin{aligned}
\sup_{j, s}\left|
\widehat{\mathcal{R}}_{\mathrm{t}}(\bar{y}_{\mathrm{t}_L}\ind{x_j \leq s} + \bar{y}_{\mathrm{t}_R}\ind{x_j > s}) - {\mathcal{R}}_{\mathrm{t}}(\bar{y}_{\mathrm{t}_L}\ind{x_j \leq s} + \bar{y}_{\mathrm{t}_R}\ind{x_j > s})\right| \rightarrow 0 \quad \quad a.s.\\
\end{aligned}
\]
Therefore, asymptotically, CART selects at each step the binary split of node t that minimizes the prediction risk among all admissible splits. However, the above arguments break down when the sample size is limited or when either the tree depth or the covariate dimensionality increases with the sample size. In such settings, CART's empirical $L_2$ risk minimization strategy may overfit the training data and fail to produce trees with minimal prediction risk, while alternative splitting criteria may have better generalization performance.

To illustrate the issues discussed above, consider training data $\{(\vecX_i, y_i)\}_{i=1}^N$ sampled i.i.d. from the linear model 
\[y = 10\beta x_{1} +  8\beta x_{2} + 6\beta x_{3} + \epsilon,
\]
where $\vecX = (x_1, \ldots, x_5)$ is uniformly distributed on $(0, 1]^5$. Among these covariates, $x_4$ and $x_5$ are purely noise variables that do not influence the response. The error $\epsilon$ is normally distributed with zero mean and standard deviation $10 \cdot x_3$. We use Figure \ref{fig:overfit_leaves} to compare how CART and our proposed covariance-driven tree method, CovRT, perform on both training and testing data. The detailed procedures for both CART and CovRT are given in Algorithm \ref{alg:tree}. 
Regression trees are first fitted on a training sample of size $N = 3000$ and grown to the maximum depth. The trees are then pruned using weakest-link pruning to obtain a prespecified number of leaves. The generalization performance of both methods is evaluated on an independent testing sample of size 3000.

\begin{algorithm}[H]
    \renewcommand{\algorithmicrequire}{\textbf{Input:}}
    \renewcommand{\algorithmicensure}{\textbf{Output:}}
    \caption{Algorithms for \textbf{CART} and  \textbf{CovRT}}
    \small
    \begin{algorithmic}[1]
        \Require
        Dataset $\mathcal{D} = \{(\mathbf{x}_1, y_1), \dots, (\mathbf{x}_N, y_N)\}$ with covariates $\mathbf{x}_i \in \mathbb{R}^p$, response $y \in \mathbb{R}$, 
        maximum tree depth $K$, 
        minimum node size $n_{\min}$ (a typical default value is 5 for regression).
        \Ensure
        Regression tree $\widehat T_K$ and prediction function $\hat{g}(\widehat T_K)(\mathbf{x})$.

        \State Initialize tree $\widehat T_K$ with the root node being the entire covariate space.
        \For{$k = 1, \ldots, K$}
        \For{each terminal node t in $\widehat T_K$}
            \If{node t has sample size $N_{\mathrm{t}} \leq n_{\min}$}
              \State \textbf{Continue}
            \Else
            \State Compute candidate splits:
            \For{each covariate $j \in \{1, \dots, p\}$}
                \State Sort unique values of $x_j$ in t: $v_1 < v_2 < \dots < v_q$.
                \For{each threshold $s \in \{\frac{v_l + v_{l+1}}{2} \mid l= 1,\dots,q-1\}$}
                    \State Partition t into left $\mathrm{t}_L = \{\vecX \in \mathrm{t} : x_j \leq s\}$ and right $\mathrm{t}_R = \mathrm{t} \setminus \mathrm{t}_L$.
                    \State Compute splitting criterion $\Delta_j^{(s)}$:
                    \State For  \textbf{CART}: $\Delta_j^{(s)} = \mathcal{I} \mathcal{G}(j, s, \mathrm{t})$ defined in (\ref{eq:impurity_gain}). 
                    \State For  \textbf{CovRT}: $\Delta_j^{(s)} = \widehat{\mathcal{CS}}(j, s, \mathrm{t})$ defined in (\ref{eq:heter_gain_feasible}).
                \EndFor
            \EndFor
            \State Select optimal split $(j^*, s^*)$ with maximal $\Delta_j^{(s)}$.
            \State Split node t into daughter nodes $\mathrm{t}_L$ and  $\mathrm{t}_R$ using $(j^*, s^*)$.
         \EndIf
        \EndFor
        \EndFor
       \State Optionally prune tree $\widehat{T}_K$ using a validation set or cross-validation.
        \State \Return $\widehat T_K$ and $\hat{g}(\widehat T_K)(\mathbf{x}) = \bar{y}_{\mathrm{t}} \  \text{for} \ \mathbf{x} \in \mathrm{t}$.
    \end{algorithmic}
    \label{alg:tree}
\end{algorithm}

 In Figure~\ref{fig:overfit_leaves}, we observe that when the number of leaves is small, the generalization gaps of both trees, measured by the difference between their performance on the training and testing data, are negligible. As tree depth increases and the number of leaves grows, the generalization gaps widen. When the number of leaves exceeds 12, the prediction risk on the testing data begins to increase, indicating overfitting. Notably, our CovRT exhibits a smaller generalization gap than CART, with higher prediction risk on the training data but better performance on the testing data, indicating superior generalization performance and robustness to overfitting.


\begin{figure}[H]
    \centering
    \includegraphics[width = 0.6\textwidth]{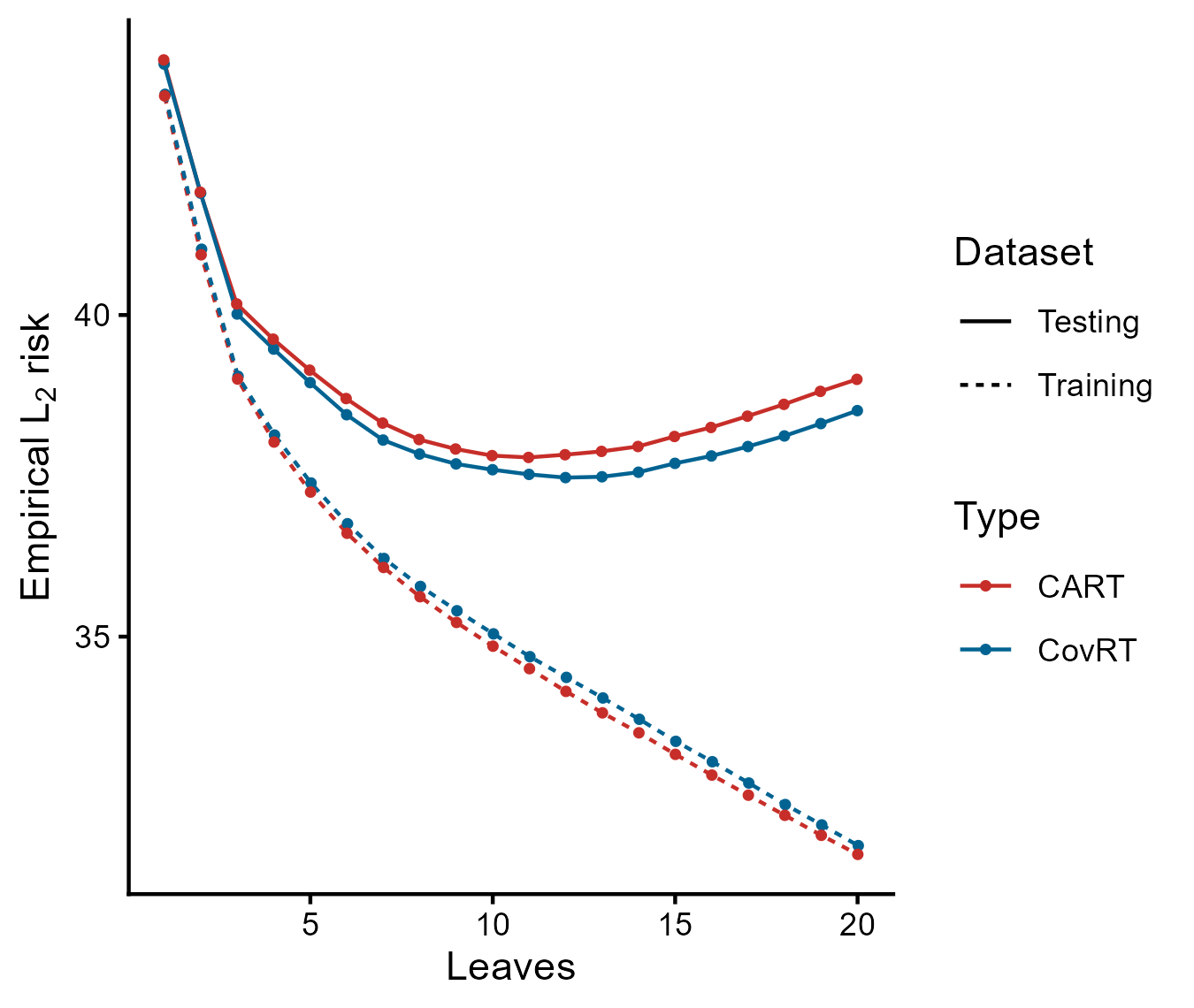}
    \caption{Empirical $L_2$ risk for CART (red) and CovRT (blue) with the number of leaves ranging from 1 to 20, evaluated on the training data (dashed lines) and the testing data (solid lines). Points indicate averages over 500 replications.} 
    \label{fig:overfit_leaves}
\end{figure}

\section{Statistical Properties of CovRT}
\subsection{Oracle Inequality and Consistency}
In this section, we establish the consistency of CovRT, as defined in Algorithm~\ref{alg:tree}, in high-dimensional settings where the complexity of the regression function $g^*(\vecX)$ increases with the sample size. We obtain and oracle inequality and prove that CovRT remains consistent for additive $g^*(\vecX)$, even when the dimension of the covariates grows sub-exponentially with the sample size. A similar consistency result for CART has been established by \citet{klusowski2024}.

The remainder of this section is organized as follows. First, we show the relationship between our splitting criterion $\widehat{\mathcal{CS}}(\mathrm{t})$ and the excess empirical risk  $\widehat{\mathcal{R}}_{\mathrm{t}}\left(\hat{g}\left(\widehat T^\text{Cov}_{K-1}\right)\right)-\widehat{\mathcal{R}}_{\mathrm{t}}(g)$ at each node t, for any candidate function $g$. Then, we establish the empirical risk bound and an oracle inequality that bounds the expected prediction risk over the data $\mathcal{D}$. Oracle inequalities are often used in nonparametric statistics and high-dimensional prediction because they hold for finite sample sizes and adapt to the unknown complexity of the underlying model, such as its sparsity or smoothness \citep{lederer2019oracle}. Finally, we discuss sufficient conditions under which CovRT is consistent.

 \begin{lemma}[Excess empirical risk within node]
 Let $g(\cdot) \in \mathcal{G}^1$ and $K \geq 1$ be any depth. Then for any terminal node t of the tree $\widehat T^\text{Cov}_{K-1}$ such that $\widehat{\mathcal{R}}_{\mathrm{t}}\left(\hat{g}\left(\widehat T^\text{Cov}_{K-1}\right)\right)>$ $\widehat{\mathcal{R}}_{\mathrm{t}}(g)$, we have
 
 $$
 \widehat{\mathcal{CS}}(\mathrm{t}) \geq \frac{\left\{\widehat{\mathcal{R}}_{\mathrm{t}}\left(\hat{g}\left(\widehat T^\text{Cov}_{K-1}\right)\right)-\widehat{\mathcal{R}}_{\mathrm{t}}(g)\right\}^2}{4\|g\|_{\mathrm{TV}}^2},
 $$
 where $\widehat{\mathcal{CS}}(\mathrm{t})$ is the splitting criterion for CovRT defined in (\ref{eq:heter_gain_feasible}) and $\|\cdot\|_{\mathrm{TV}}$ is the total variation $\ell_1$ norm defined in (\ref{eq:TV_norm}).
 \label{lemma:heter}
  \end{lemma}

  Lemma \ref{lemma:heter} gives the lower bound of the splitting criterion $\widehat{\mathcal{CS}}(\mathrm{t})$ in terms of within-node excess empirical risk for CovRT. Utilizing Lemma \ref{lemma:heter}, we can establish the upper bound on the global excess empirical risk $\widehat{\mathcal{R}}\left(\hat{g}\left(\widehat T^\text{Cov}_{K}\right)\right)-\widehat{\mathcal{R}}(g)$:
  \begin{theorem}[Empirical risk bound]
       For depth $K$ tree $\widehat T^\text{Cov}_{K}$,
       \[
    \widehat{\mathcal{R}}\left(\hat{g}\left(\widehat T^\text{Cov}_{K}\right)\right) \leq \inf _{g(\cdot) \in \mathcal{G}^1}\left\{\widehat{\mathcal{R}}(g)+\frac{\|g\|_{\mathrm{TV}}^2}{K+3}\right\}.
       \]
       \label{thm:empirical_risk_bound}
  \end{theorem}
   To derive the oracle inequality, we assume that the error $\varepsilon=y-g^*(\mathbf{x})$ is sub-Gaussian, i.e., there exists $\sigma^2>0$ such that for all $u \geq 0$,
\begin{equation}
\Prob{|\varepsilon| \geq u} \leq 2 \exp \left(-u^2 /\left(2 \sigma^2\right)\right).
\label{eq:subGaussian}
\end{equation}
Then we have the following result.
\begin{theorem}[Oracle inequality] Let $K \geq 1$ be any depth. With the sub-Gaussian noise (\ref{eq:subGaussian}), we have
\begin{equation}
\mathbb{E}_{\mathcal{D}}\left(\left\|\hat{g}\left(\widehat T^\text{Cov}_K\right)-g^*\right\|^2\right) \leq 2 \inf _{g(\cdot) \in \mathcal{G}^1}\left\{\left\|g-g^*\right\|^2+\frac{\|g\|_{\mathrm{TV}}^2}{K+3}+C_1 \frac{2^K \log ^2(N) \log (N p)}{N}\right\},
\label{eq:oracal}
\end{equation}
where the definition of norm $\|\cdot\|$ is given in (\ref{eq:norm}) and $C_1$ is a positive constant that depends only on $\|g\|_{\infty}:=\sup _{\mathbf{x}}|g(\mathbf{x})|$ and $\sigma^2$. 
\label{thm:oracle}
\end{theorem} 

 The oracle inequality established in Theorem \ref{thm:oracle} takes the same form as the one in Theorem 4.3 of \cite{klusowski2024} for CART. This equivalence implies that, in theory, although CovRT does not explicitly target the minimization of empirical $L_2$ risk as CART does, its prediction accuracy is not compromised and can even be better. Theorem \ref{thm:oracle} immediately implies the following high-dimensional consistency for CovRT.

\begin{corollary}[Consistency] Consider a sequence of true models $\left\{g_N^*(\cdot)\right\}_{N=1}^{\infty}$, where $g_N^*(\mathbf{x})=\sum_{j=1}^{p_N} g_j\left(x_j\right) \in \mathcal{G}^1$ and $\sup _N\left\|g_N^*\right\|_{\infty}<$ $\infty$. Suppose that $K_N \rightarrow \infty,\left\|g_N^*\right\|_{\mathrm{TV}}=o\left(\sqrt{K_N}\right)$, and ${2^{K_N} \log ^2(N) \log \left(N p_N\right)}/{N} \rightarrow 0$ as $N \rightarrow \infty$. With the sub-Gaussian noise (\ref{eq:subGaussian}), CovRT is consistent, meaning that

$$
\lim _{N \rightarrow \infty} \mathbb{E}_{\mathcal{D}}\left(\left\|\hat{g}\left(\widehat T^\text{Cov}_{K_N}\right)-g_N^*\right\|^2\right)=0.
$$
\label{cor:consistency}
\end{corollary}

 We want to emphasize here, as previously noted,  $\Vert g^*_N\Vert_{\text{TV}} \leq \Vert g^*_N\Vert_{\ell_0}\cdot \max_{j \leq p_N}\text{TV}(g_j)$. If we assume that $\max_{j \leq p_N}\text{TV}(g_j)$ is bounded, then Corollary \ref{cor:consistency} holds for high-dimensional sparse models \citep{belloni2013least} where the number of covariates grow sub-exponentially (i.e., $\log(p_N) = O(N^{1-\xi})$) and the model is sparse of order $o(\sqrt{ \log _2(N)})$. This can be achieved by choosing $K_N=\left\lfloor(\xi / 2) \log _2(N)\right\rfloor$ for $\xi \in(0,1)$. As discussed in \cite{klusowski2024}, the high-dimensional consistency established in Corollary~\ref{cor:consistency}, even when the dimensionality grows sub-exponentially with the sample size, cannot be achieved by non-adaptive methods such as conventional multivariate kernel regression (e.g., Nadaraya--Watson or local polynomial regression with a common bandwidth across all directions) or $k$-nearest neighbors  with Euclidean distance. These methods perform local estimation using data points that are close in every coordinate direction and do not adjust the amount of smoothing along each dimension according to how much the predictor variable affects the response variable. As a result, they are prone to the curse of dimensionality, even in the presence of model sparsity. 

 Rather than selecting an optimal tree depth, one can first grow a full tree $T_{max}$ to the maximum depth. Such a tree typically overfits the training data. Cost-complexity pruning is then performed by minimizing the penalized empirical risk
 \[
\widehat{\mathcal{R}}\left(\hat{g}\left( T\right)\right) + \alpha \cdot |T|
\]
over all subtrees $T \subset T_{max} $ that can be obtained by collapsing internal nodes of $T_{max}$, where $|T|$ denotes the number of terminal nodes \citep{breiman1984, gey2005model}. This optimization can be efficiently implemented via weakest-link pruning and yields oracle inequalities of the form \eqref{eq:oracal}, with the infimum taken over both the tree depth $K$ and additive functions $g(\cdot) \in \mathcal{G}^1$.

 \subsection{Robustness to Overfitting}
  In the simulation study and empirical application, we further demonstrate that CovRT outperforms CART, with the performance gap particularly substantial on the Boston Housing dataset \citep{harrison1978hedonic}, where the sample size is moderate and the covariate dimensionality is relatively high. Such empirical findings may be intuitively explained by noticing that, at each node t,
 \[
 \widehat{\mathcal{CS}}(j, s, \mathrm{t}) = \widehat{P}_{\mathrm{t}_L}\widehat{P}_{\mathrm{t}_R} \left[ \frac{1}{N_{\mathrm{t}}} \sum_{\mathbf{x}_i \in \mathrm{t}} \left(y_i - \bar{y}_{\mathrm{t}}\right)^2 - \frac{1}{N_{\mathrm{t}}}\left\{\sum_{\mathbf{x}_i \in \mathrm{t}_L} \left(y_i - \bar{y}_{\mathrm{t}_L}\right)^2 + \sum_{\mathbf{x}_i \in \mathrm{t}_R} \left(y_i - \bar{y}_{\mathrm{t}_R}\right)^2\right\}\right],
 \]
  which can be viewed as a penalized version of CART's empirical $L_2$ risk minimization criterion. The factor $\widehat{P}_{\mathrm{t}_L}\widehat{P}_{\mathrm{t}_R}$ in $\widehat{\mathcal{CS}}(j, s, \mathrm{t})$ serves as a natural penalty on highly unbalanced splits. When a split creates one large daughter node and one very small one, this product becomes small, reducing the overall gain and discouraging such divisions. This mechanism prevents the model from overfitting to small, potentially noisy subsets of the data, and instead encourages more balanced and stable splits. As a result, CovRT is more robust to overfitting and produces trees that generalize better to unseen data.
  
  Here is a simple illustration of the above issues, motivated by Example 5 in \cite{ishwaran2015effect}. Consider training data $\{(x_{1i}, y_i)\}_{i = 1}^N$ of size $N = 200$ sampled i.i.d. from the simple linear model:
  \begin{equation}
    y_i = c_0 + c_1x_{1i} + \epsilon_i,\ i = 1, \ldots, N,
    \label{eq:simple_linear}
  \end{equation}
  where $x_{1i}$ is uniformly distributed on $(0, 1]$ and $\epsilon_i$ follows a standard normal distribution. We consider three scenarios: (a) noisy ($c_0 = 1$, $c_1 = 0$); (b) medium signal ($c_0=1$, $c_1=0.5$); and (c) strong signal ($c_0 = 1$, $c_1 = 1$). We fit depth-1 trees on the training data and investigate the distribution of split points $\hat{s}$ generated by CART and CovRT. For comparison, we also include a purely random splitting tree, where the split point is selected purely at random.

  The estimated density of the split points $\hat{s}$ from 5000 simulations is shown in Figure~\ref{fig:densities} for the three tree methods. In the noisy scenario, we see that the density of CART split points concentrates near the edges. This phenomenon, known as the end-cut preference splitting property \citep[Chapter~11.8]{breiman1984}, leads CART to generate highly unbalanced splits where $\widehat{P}_{\mathrm{t}_L}$ is close to zero or one when splitting on noisy covariates, and has generally been considered an undesirable property. In the medium signal scenario, the density of CART split points resembles that of purely random splitting, while in the strong signal scenario, the density concentrates near the midpoint, resulting in more stable and balanced splits. On the other hand, CovRT does not exhibit such end-cut preference and instead tends to produce more balanced and stable splits, with split points increasingly concentrated near the midpoint as the signal strength grows. 
        \begin{figure}[H]
    \centering
    \includegraphics[width = 0.8\textwidth]{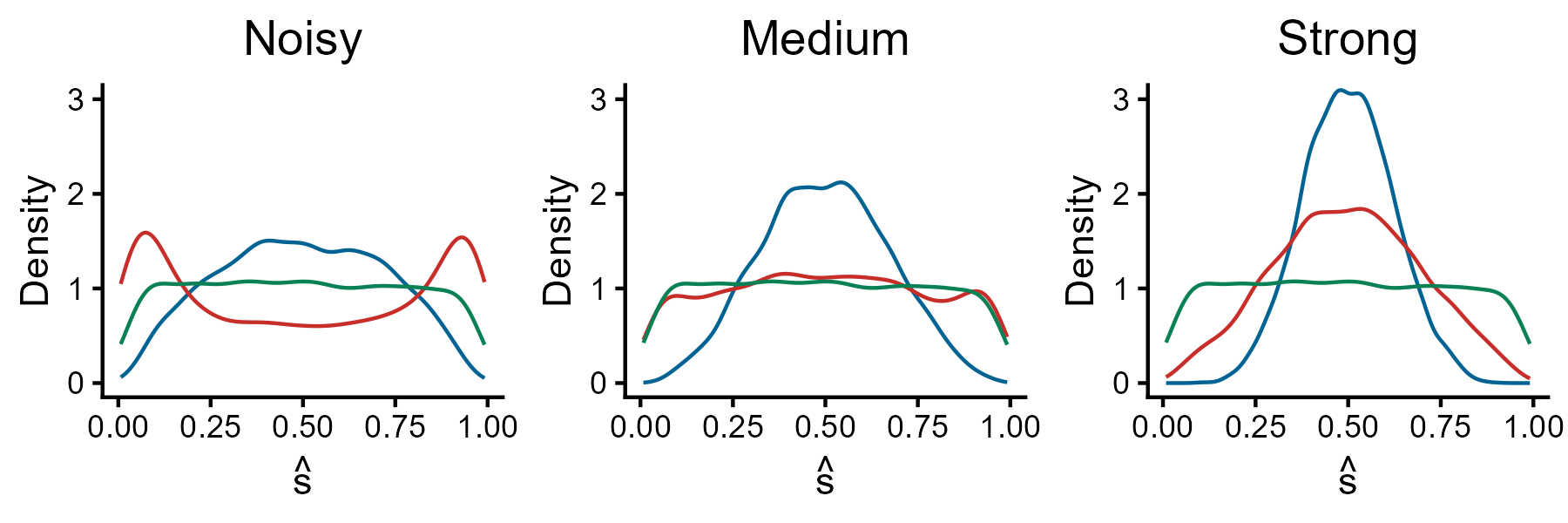}
    \caption{Density for split points $\hat{s}$ for purely random splitting tree (green), CART (red) and CovRT (blue) where $g^* = c_0 + c_1x_{1}$ for $c_0 = 1, c_1 = 0$ (left: noisy), $c_0 = 1, c_1 = 0.5$ (middle: medium signal) and $c_0 = 1, c_1 = 1$ (right: strong signal).}
    \label{fig:densities}
\end{figure}
    
 A more intriguing finding is that incorporating the factor $\widehat{P}_{\mathrm{t}_L}\widehat{P}_{\mathrm{t}_R}$ not only encourages more balanced splits but also helps the tree more effectively identify covariates with true signals. We continue to consider the data-generating process in (\ref{eq:simple_linear}), but now include four additional noisy covariates $(x_2, x_3, x_4, x_5)$, which are independently and uniformly distributed on $(0, 1]^4$ and independent of both $x_1$ and $y$. Figure~\ref{fig:choose} displays the proportion of times that depth-1 trees correctly selected the true signal covariate $x_1$ for splitting in 5000 simulations. For the purely random splitting tree, both the covariate and the split point are chosen entirely at random. We see that CovRT achieves higher accuracy in identifying the true signal covariate than CART across all signal strengths. For example, in the medium signal setting where the signal strength equals 0.5, CovRT achieves an accuracy of 0.643 in correctly splitting on the true signal covariate $x_1$, compared to 0.588 for CART. The difference of 0.055 is statistically significant, with a p-value less than $10^{-8}$ from the one-sided two-proportion Z-test. This improvement can be explained by Figure \ref{fig:densities}: CART exhibits an end-cut preference for noisy covariates, while the inclusion of the factor $\widehat{P}_{\mathrm{t}_L}\widehat{P}_{\mathrm{t}_R}$ in CovRT penalizes splits near edges. Consequently, CovRT discourages splitting on noisy covariates and instead favors covariates with true signal, where the optimal split points tend to lie closer to the midpoint. The higher accuracy of CovRT in identifying true signal compared to CART can translate into a more efficient growth of tree branches, resulting in shallower trees that are less prone to overfitting.

  \begin{figure}[H]
    \centering
    \includegraphics[width = 0.5\textwidth]{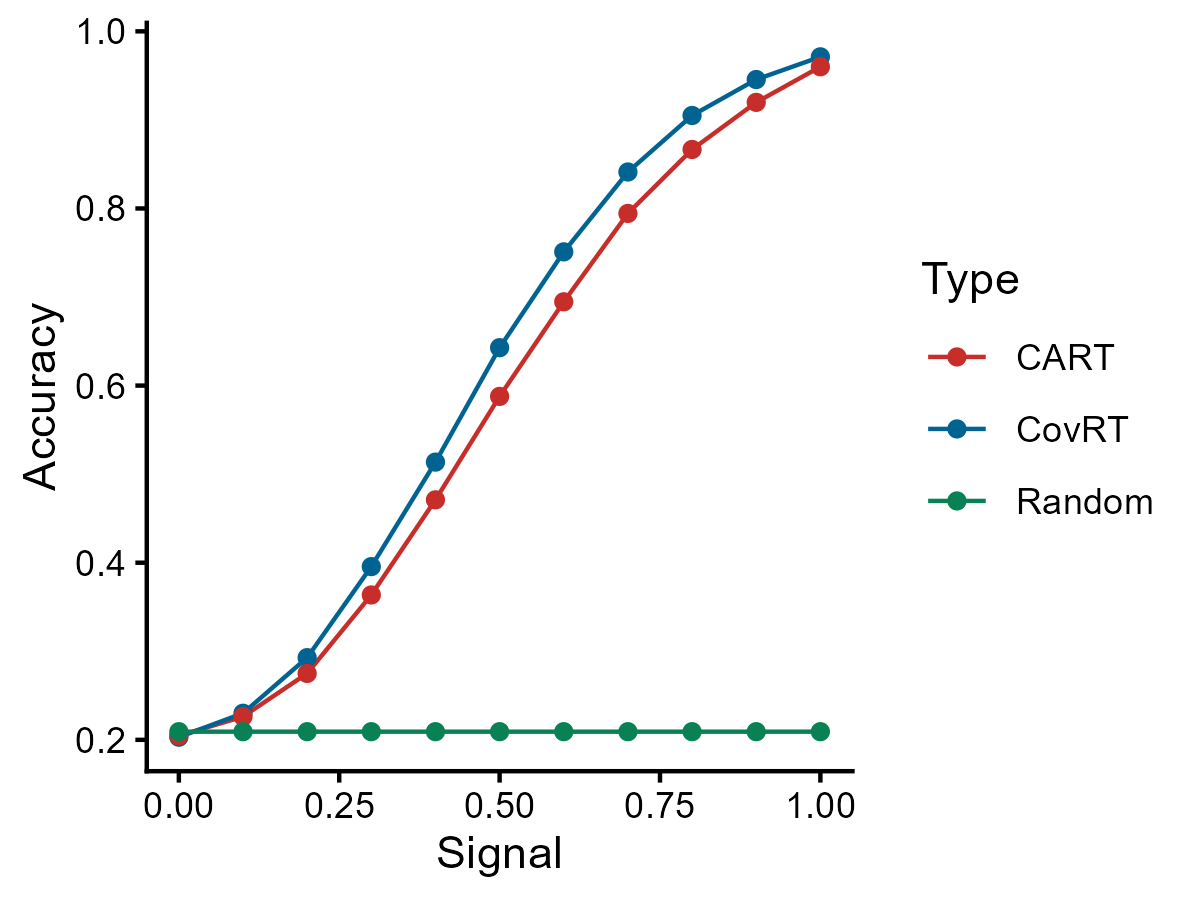}
    \caption{Accuracy for purely random splitting tree (green), CART (red), and CovRT (blue) in correctly splitting on the true signal covariate $x_1$.}
    \label{fig:choose}
\end{figure}


\section{Simulation Study}
In this section, we analyze the finite-sample behavior of CovRT compared with CART. For training data generated from four different data-generating processes, we will fit CovRT with fixed depths, CART with fixed depths, post-pruned CovRT, and post-pruned CART with cost-complexity pruning. 
The prediction accuracy of the three algorithms will be evaluated by their empirical $L_2$ risk on the testing data of size 1000. For algorithms CovRT with fixed depths and CART with fixed depths, results are presented by varying the choice of maximum depth $K$.The post-pruned CovRT and CART are fitted using the optimal complexity parameter $\alpha$, which is selected based on an independent validation set of the same size as the training data. The four different data-generating processes are given as follows:

\textbf{Model 1}:
\[
 y_i = \sum_{j = 1}^{10}\beta_jx_{ij} + \epsilon_i,
\]
where $\vecX_i = (x_{1i}, \ldots, x_{10i}) \sim \text{Unif}((0, 1]^{10})$ and $\epsilon_{i} \sim \mathcal{N}(0, 2^2)$. Let $(\beta_{1}, \beta_{2}, \beta_{3}, \beta_{4}) = (10, 8, 6, 2)$ and $\beta_j = 0$ for all $j \geq 5$. 

\textbf{Model 2}: 
\[
 y_i = \sum_{j = 1}^{10}\beta_jx^2_{ji} + \epsilon_i,
\]
where $\vecX_i = (x_{1i}, \ldots, x_{10i}) \sim \text{Unif}((0, 1]^{10})$ and $\epsilon_{i} \sim \mathcal{N}(0, 2^2)$. Let $(\beta_{1}, \beta_{2}, \beta_{3}, \beta_{4}) = (10, 8, 6, 2)$ and $\beta_j = 0$ for all $j \geq 5$. 

\textbf{Model 3}: 
\[
 y_i = \beta_1x_{1i}  + \beta_2x_{2i} 
  + \beta_3\cdot\ind{x_{3i} > 0.5}+ \beta_4\cdot\ind{x_{4i} > 0.6}+ \epsilon_i,
\]
where $\vecX_i = (x_{1i}, \ldots, x_{10i}) \sim \text{Unif}((0, 1]^{10})$ and $\epsilon_{i} \sim \mathcal{N}(0, 2^2)$. Let $(\beta_{1}, \beta_{2}, \beta_{3}, \beta_{4}) = (6, 10, 8, 4)$. 

\textbf{Model 4}: 
\[
 y_i = \beta_1x_{1i}\cdot\ind{x_{1i} > 0.5}  + \beta_2\sqrt{x_{2i}}
  + \beta_3\sin{(0.5\pi\cdot x_{3i})}+ \beta_4\cos{(\pi\cdot x_{4i})}+ \epsilon_i,
\]
where $\vecX_i = (x_{1i}, \ldots, x_{10i}) \sim \text{Unif}((0, 1]^{10})$ and $\epsilon_{i} \sim \mathcal{N}(0, 2^2)$. Let $(\beta_{1}, \beta_{2}, \beta_{3}, \beta_{4}) = (6, 10, 8, 4)$.

\begin{table}[]
\caption{The prediction risk for CovRT with fixed depth, CART with fixed depth, and post-pruned CART with cost-complexity pruning, evaluated by the empirical $L_2$ risk on the testing data. The bold font indicates the minimal $L_2$ risk}
\centering
\begin{tabular}{llcccc}
\hline
      &         & \multicolumn{4}{c}{Algorithms}       \\ \cline{3-6} 
Model & Depth $K$ & CovRT & CART & Post-pruned CovRT & Post-pruned CART   \\ \hline
1     & 3       & 9.23 & 9.58 &  \multirow{4}{*}{\textbf{8.14}} &  \multirow{4}{*}{8.37} \\
      & 4       & 8.23 & 8.65    &                    \\
      & 5       & 8.31 & 8.55   &                    \\
      & 6       & 8.62 & 8.74    &                    \\
      &         &       &      &                    \\
2     & 3       & 9.19 & 9.40  & \multirow{4}{*}{\textbf{7.99}} & \multirow{4}{*}{8.19}                  \\
      & 4       & 8.01 & 8.39   &                  \\
      & 5       & 8.21 & 8.34   &   
               \\
      & 6       & 8.55 & 8.54   &    \\
      &         &       &      &                    \\
3     & 3       & 5.62  & 5.83 & \multirow{4}{*}{\textbf{5.53} } & \multirow{4}{*}{5.66 }                              \\
      & 4       & 5.62 & 5.84   &   
               \\
      & 5       & 6.28  & 6.31   &    \\ 
    & 6       & 6.72  & 6.70   &    \\ 
      &         &       &      &                    \\
4     & 3       & 14.41  & 14.91 & \multirow{4}{*}{\textbf{10.61} } & \multirow{4}{*}{10.91} \\
& 4       & 11.07 & 11.69 &                  \\
      & 5       & 10.70 & 11.13   & 
      \\
      & 6       & 10.90 & 11.18   &  
               \\
      \hline
\end{tabular}
\label{tab:sim_L2}
\end{table}

We analyze the prediction accuracy of CovRT with fixed depth, CART with fixed depth, post-pruned CovRT, and post-pruned CART, using training data of fixed size 300 generated from the four models described above. Table \ref{tab:sim_L2} presents the empirical $L_2$ risk of the four algorithms evaluated on the testing data. Results are averaged over 500 replications of the data. Across all four data-generating processes, CovRT outperforms CART in prediction accuracy, and post-pruned CovRT achieves the lowest prediction risk among the four methods.

\section{Empirical Applications}

We illustrate the performance of CovRT on three real-world benchmark datasets:

\begin{itemize}
    \item \textit{Boston Housing} considers housing values in the area of Boston, Massachusetts \citep{harrison1978hedonic}. It consists of 506 observations and 14 variables, with the target variable being the median home value $y$ in a neighborhood, and features $\vecX$ including nitrogen oxide concentration, crime rate, average number of rooms, among others.

    \item \textit{Airfoil Self-Noise} consists of data from different-sized NACA 0012 airfoils at various wind tunnel speeds and angles of attack \citep{brooks1989airfoil}. It includes 1,503 observations and 6 variables, with the target variable being the scaled sound pressure $y$, and features $\vecX$ including frequency and angles of attack, among others.

    \item \textit{Abalone} aims to predict the age of abalone from physical measurements \citep{nash1994population}. It contains 4,177 observations and 9 variables, with the target variable being the number of rings $y$ (+1.5 gives the age in years), and features $\vecX$ including sex, length, and height, among others.
\end{itemize}

For each dataset, we randomly partition the data into training, validation, and testing sets in a 2:1:1 ratio. On each training set, we fit four tree models: fixed-depth CovRT, fixed-depth CART, post-pruned CovRT, and post-pruned CART. Prediction performance is evaluated on the testing set using empirical $L_2$ risk and $R^2$. Results are averaged over 100 random partitions and summarized in Table~\ref{tab:reg_tree_performance}.

\begin{table}[]
  \centering
  \begin{threeparttable}
  \caption{The prediction risk and $R^2$ for various tree-based algorithms on the Boston Housing, Airfoil Self-Noise, and Abalone datasets}
  \label{tab:reg_tree_performance}
  \sisetup{table-format=1.3}  
\small
  \begin{tabular}{l *{6}{S}}
    \toprule
    & \multicolumn{2}{c}{\textbf{Boston Housing}} 
    & \multicolumn{2}{c}{\textbf{Airfoil Self-Noise}} 
    & \multicolumn{2}{c}{\textbf{Abalone}} \\
    \cmidrule(lr){2-3} \cmidrule(lr){4-5} \cmidrule(lr){6-7}
    \textbf{Algorithm} 
    & {$L_2$ risk} & {$R^2$} 
    & {$L_2$ risk} & {$R^2$} 
    & {$L_2$ risk} & {$R^2$} \\
    \midrule
    Fixed-depth CovRT      & 21.07 & 0.755 & 11.97 & 0.752 & 5.319 & 0.493 \\
    Fixed-depth CART        & 22.95 & 0.734 & 12.02 & 0.751 & 5.482 & 0.478 \\
    & & & & & &  \\
   Post-pruned CovRT     & 20.95 & 0.756 & 11.98 & 0.752 & 5.412 & 0.484 \\
    Post-pruned CART   & 22.81 & 0.736 & 12.03 & 0.751 & 5.513 & 0.476 \\
    \bottomrule
  \end{tabular}
  \end{threeparttable}
\end{table}

From Table~\ref{tab:reg_tree_performance}, we observe that CovRT exhibits better predictive performance than CART on three datasets, both for fixed-depth trees and post-pruned trees. Notably, on the Boston Housing dataset, CovRT achieves an average prediction risk approximately 8\% lower than that of CART across 100 random data partitions, with improvements of up to about 20\% in some partitions. This substantial improvement may be attributed to the moderate sample size (fewer than 300 training observations) and the relatively high covariate dimensionality. 

\section{Discussion} \label{section:discussion}

In this study, we propose a covariance-driven splitting criterion as an overfitting robust method for regression tree growth. We establish and oracle inequality and the high-dimensional consistency of our proposed CovRT and show that, in theory, its prediction accuracy is comparable to that of CART, with the potential to produce more balanced and stable splits, thereby achieving better generalization. Through simulations and empirical applications, we demonstrate that CovRT has superior predictive performance compared to CART.

Our findings have broader implications in several aspects. First, this paper focuses on regression trees. Extending our covariance-driven splitting criterion to classification trees is left for future research.  Second, our CovRT can serve as a new building block for other algorithms such as random forests \citep{friedman2001greedy}, gradient boosted decision trees \citep{breiman2001random}, among others. Third, our work may be further developed to help infer causal effects in experimental or observational studies \citep{athey2016recursive, wager2018estimation, cattaneo2022pointwise, tabord2023}.


\bibliographystyle{apalike}
\bibliography{Inter}

\begin{flushleft}
{\huge\bfseries Appendix}
\end{flushleft}

\appendix


\section{Additional theoretical results}
For multivariate functions $u(\vecX_i)$ and $v(\vecX_i)$, define the squared norm and inner product over the dataset $\mathcal{D}$ as
\[
\|u\|_{\mathcal{D}}^2:=\frac{1}{N} \sum_{i=1}^N\left(u\left(\mathbf{x}_i\right)\right)^2 \quad \text { and } \quad\langle u, v\rangle_{\mathcal{D}}:=\frac{1}{N} \sum_{i=1}^N u\left(\mathbf{x}_i\right) v\left(\mathbf{x}_i\right).
\]
Similarly, define the within-node squared norm and inner product, respectively, as
\[
\|u\|_{\mathrm{t}}^2:=\frac{1}{N_{\mathrm{t}}} \sum_{\mathbf{x}_i \in \mathrm{t}}\left(u\left(\mathbf{x}_i\right)\right)^2 \quad \text { and } \quad\langle u, v\rangle_{\mathrm{t}}:=\frac{1}{N_{\mathrm{t}}} \sum_{\mathbf{x}_i \in \mathrm{t}} u\left(\mathbf{x}_i\right) v\left(\mathbf{x}_i\right).
\]
For a splitting at node t with covariate $x_j$ and splitting value $s$, we define 
\[
\Phi_{\mathrm{t}}(\mathbf{x}):=\beta_0+\beta_1 \ind{x_j>s} \in \mathcal{G}^1, \quad \beta_0:=\frac{N_{\mathrm{t}_R}}{N_{\mathrm{t}}}, \quad \beta_1:=-1. 
\]
Then we have 
\begin{equation}
\Phi_{\mathrm{t}}(\mathrm{x}) \mathbf{1}(\mathrm{x} \in \mathrm{t})=\frac{\ind{\mathrm{x} \in \mathrm{t}_L} N_{\mathrm{t}_R}-\ind{\mathrm{x} \in \mathrm{t}_R} N_{\mathrm{t}_L}}{N_{\mathrm{t}}} .
\label{eq:stump_func}
\end{equation}

The following result will be used in the proof of Lemma 1.

\begin{proposition}
    Let t be any node in $\widehat T^\text{Cov}$. Then, the feasible splitting criterion for CovRT satisfies
    \[
    \widehat{\mathcal{CS}}(j, s, \mathrm{t})=\left|\left\langle y-\bar{y}_{\mathrm{t}}, \Phi_{\mathrm{t}}\right\rangle_{\mathrm{t}}\right|^2 .
    \]
    \label{prop:inner_product}
\end{proposition}
\begin{proof}
    \[
    \begin{aligned}
    \left|\left\langle y-\bar{y}_{\mathrm{t}}, \Phi_{\mathrm{t}}\right\rangle_{\mathrm{t}}\right|^2 &=\left| \frac{1}{N_{\mathrm{t}}} \sum_{\mathbf{x}_i \in \mathrm{t}} (y_i - \bar{y}_{\mathrm{t}}) (\beta_0+\beta_1 \ind{x_{ij}>s})\right|^2\\
    &=
    \left|\frac{N_{\mathrm{t}_L} \left(\bar{y}_{\mathrm{t}_L} - \bar{y}_{\mathrm{t}}\right) N_{\mathrm{t}_R}-N_{\mathrm{t}_R}\left( \bar{y}_{\mathrm{t}_R} - \bar{y}_{\mathrm{t}_L}\right) N_{\mathrm{t}_L}}{N_{\mathrm{t}}^2 }\right|^2=\frac{N_{\mathrm{t}_L}^2 N_{\mathrm{t}_R}^2}{N_{\mathrm{t}}^4}\left(\bar{y}_{\mathrm{t}_L}-\bar{y}_{\mathrm{t}_R}\right)^2=\widehat{\mathcal{CS}}(j, s, \mathrm{t}).
    \end{aligned}
    \]
\end{proof}

\section{Proof of the main results}

        \subsection{Proof of Theorem 1}
\begin{proof}
     By assumptions, we assume an additive regression model $\mathbb{E}(y \mid \vecX) = g^*(\vecX) = g_1\left(x_1\right)+g_2\left(x_2\right)+\cdots+g_p\left(x_p\right)$ and that the covariates are independent $\vecX \sim \prod_{j=1}^pF_{\vecX_j}$. 
    Then, at each node $\mathrm{t} = [a_1, b_1] \times [a_2, b_2] \times \cdots \times [a_p, b_p]$, with $s \in [a_j, b_j]$, $\mathrm{t}_L = \{\vecX \in \mathrm{t} : x_j \leq s\}$ and $\mathrm{t}_R = \{\vecX \in \mathrm{t} : x_j > s\}$,
\begin{equation}
\begin{aligned}
\mathbb{E}(y\mid \vecX \in \mathrm{t}_L) - \mathbb{E}(y\mid \vecX \in \mathrm{t}_R) 
&= \mathbb{E}(y\mid x_j \leq s, \vecX \in \mathrm{t}) - \mathbb{E}(y\mid x_j > s, \vecX \in \mathrm{t})\\
&= \mathbb{E}\{g_j\left(x_j\right)\mid x_j \leq s, \vecX \in \mathrm{t}\} - \mathbb{E}\{g_j\left(x_j\right)\mid x_j > s, \vecX \in \mathrm{t}\}\\
&= \mathbb{E}\{g_j\left(x_j\right)\mid a_j \leq x_j \leq s\} - \mathbb{E}\{g_j\left(x_j\right)\mid s < x_j \leq b_j\}\\
&= \frac{\mathbb{E}\{g_j\left(x_j\right)\ind{a_j \leq x_j \leq s}\}}{F_{\vecX_j}((a_j,s])} - \frac{\mathbb{E}\{g_j\left(x_j\right)\ind{ s < x_j \leq b_j}\}}{F_{\vecX_j}((s,b_j])} \\
&= \frac{\int_{a_j}^sg_j\left(x\right)dF_{\vecX_j}}{F_{\vecX_j}((a_j,s])} - \frac{\int_{s}^{b_j}g_j\left(x\right)dF_{\vecX_j}}{F_{\vecX_j}((s,b_j])}. \\
\end{aligned}
\end{equation}

So 
\begin{equation}
\begin{aligned}
\mathcal{CS}(j, s, \mathrm{t})
&= P_{\mathrm{t}_L}^2P_{\mathrm{t}_R}^2\left\{\mathbb{E}(y\mid \vecX \in \mathrm{t}_L) - \mathbb{E}(y\mid \vecX \in \mathrm{t}_R)\right\}^2 \\
&= \left\{\frac{F_{\vecX_j}((a_j,s])}{F_{\vecX_j}((a_j,b_j])}\right\}^2\left\{\frac{F_{\vecX_j}((s, b_j])}{F_{\vecX_j}((a_j,b_j])}\right\}^2 \left\{\frac{\int_{a_j}^sg_j\left(x\right)dF_{\vecX_j}}{F_{\vecX_j}((a_j,s])} - \frac{\int_{s}^{b_j}g_j\left(x\right)dF_{\vecX_j}}{F_{\vecX_j}((s, b_j])}\right\}^2 \\
&= \frac{1}{\{F_{\vecX_j}((a_j,b_j])\}^4}\left\{F_{\vecX_j}((s,b_j]){\int_{a_j}^sg_j\left(x\right)dF_{\vecX_j}} - F_{\vecX_j}((a_j,s]){\int_{s}^{b_j}g_j\left(x\right)dF_{\vecX_j}}\right\}^2. \\
\end{aligned}
\end{equation}

Let $\delta = \frac{F_{\vecX_j}((a_j,s])}{F_{\vecX_j}((a_j,b_j])}$, then $\delta \in [0, 1]$ and
\begin{equation}
\begin{aligned}
\mathcal{CS}(j, s, \mathrm{t})
&= \mathcal{CS}(j, s(\delta), \mathrm{t}) \\
&= \frac{1}{\{F_{\vecX_j}((a_j,b_j])\}^2}\left\{(1 - \delta){\int_{a_j}^sg_j\left(x\right)dF_{\vecX_j}} - \delta {\int_{s}^{b_j}g_j\left(x\right)dF_{\vecX_j}}\right\}^2 \\
&= \frac{1}{\{F_{\vecX_j}((a_j,b_j])\}^2}\left\{{\int_{a_j}^sg_j\left(x\right)dF_{\vecX_j}} - \delta {\int_{a_j}^{b_j}g_j\left(x\right)dF_{\vecX_j}}\right\}^2 \\
&= \frac{1}{\{F_{\vecX_j}((a_j,b_j])\}^2}\left[{\int_{a_j}^{s}\{g_j\left(x\right) - I_j\}dF_{\vecX_j}}\right]^2, \\
\end{aligned}
\label{eq:thm1_HT}
\end{equation}
where $I_j =  \{F_{\vecX_j}((a_j,b_j])\}^{-1}{\int_{a_j}^{b_j}g_j\left(x\right)dF_{\vecX_j}}$.

Note that for $x \in [a_j, b_j]$,
\begin{equation}
\begin{aligned}
g_j\left(x\right) 
&= \mathbb{E}\{g^*(\vecX) \mid x_j = x, \vecX \in \mathrm{t}\}-  \sum_{i \neq j}\mathbb{E}\{g_i(x_i) \mid \vecX \in \mathrm{t}\} \\
&=\mathbb{E}\{g^*(\vecX) \mid x_j = x, \vecX \in \mathrm{t}\} - \mathbb{E}\{g^*(\vecX)\mid \vecX \in \mathrm{t}\} + \mathbb{E}\{g_j(x_j)\mid \vecX \in \mathrm{t}\} \\
&= \mathbb{E}(y \mid x_j = x, \vecX \in \mathrm{t}) - \mathbb{E}(y\mid \vecX \in \mathrm{t}) + I_j.\\
\label{eq:thm1_gj}
\end{aligned}
\end{equation}

Substitute (\ref{eq:thm1_gj}) into (\ref{eq:thm1_HT}) and get
\begin{equation}
\begin{aligned}
\mathcal{CS}(j, s, \mathrm{t})
&=\left[ \frac{1}{F_{\vecX_j}((a_j,b_j])}{\int_{a_j}^{s}\{\mathbb{E}(y \mid x_j = x, \vecX \in \mathrm{t}) - \mathbb{E}(y\mid \vecX \in \mathrm{t})\}dF_{\vecX_j}} \right]^2 \\
&=\left[ {\int_{a_j}^{s}\{\mathbb{E}(y \mid x_j = x, \vecX \in \mathrm{t}) - \mathbb{E}(y\mid \vecX \in \mathrm{t})\}dF_{\vecX_j\mid \mathrm{t}}} \right]^2 \\
&=\frac{1}4\Bigg[ {\int_{a_j}^{s}\{\mathbb{E}(y \mid x_j = x, \vecX \in \mathrm{t}) - \mathbb{E}(y\mid \vecX \in \mathrm{t})\}dF_{\vecX_j\mid \mathrm{t}}} \\
&\hspace{125pt}-{\int_{s}^{b_j}\{\mathbb{E}(y \mid x_j = x, \vecX \in \mathrm{t}) - \mathbb{E}(y\mid \vecX \in \mathrm{t})\}dF_{\vecX_j\mid \mathrm{t}}} \Bigg]^2. \\
\end{aligned}
\end{equation}
The last equality follows from the fact that
\[
\int_{a_j}^{b_j}
\left\{
\mathbb{E}(y \mid x_j = x, \mathbf{x} \in \mathrm{t})
-
\mathbb{E}(y \mid \mathbf{x} \in \mathrm{t})
\right\}
\, dF_{x_j \mid \mathrm{t}}
= 0,
\]
since $\mathbb{E}(y \mid \mathbf{x} \in \mathrm{t})$ is the conditional mean of $y$ over the node $\mathrm{t}$.

When $g^*(\vecX) = \alpha + \sum_{j=1}^p\beta_jx_j$,

\begin{equation}
\begin{aligned}
\mathcal{CS}(j, s, \mathrm{t})
&= \frac{1}{(b_j - a_j)^2}\left[{\int_{a_j}^{s}\{g_j\left(x\right) - I_j\}dx}\right]^2 \\
&= \frac{1}{(b_j - a_j)^2}\left[{\int_{a_j}^{s}\beta_j\left\{x - \frac{a_j+b_j}{2}\right\}dx}\right]^2. \\
\end{aligned}
\end{equation}

Then, 
\begin{equation}
\begin{aligned}
\max_{s}\mathcal{CS}(j, s, \mathrm{t}) &=  \mathcal{CS}(j, {(a_j+b_j)}/{2}, \mathrm{t}) \\
&= \frac{1}{(b_j - a_j)^2}\left[{\beta_j\int_{\frac{a_j-b_j}{2}}^{0}xdx}\right]^2 \\
&= \frac{1}{64}(b_j - a_j)^2\beta_j^2.\\
\end{aligned}
\end{equation}

\end{proof}
\subsection{Proof of Theorem 2}
\begin{proof}
First notice that for each $j = 1, \cdots, p$ and $a_j < s < b_j$, $N_{\mathrm{t}_L} / N = (1/N)\sum_{i=1}^N\ind{\mathbf{x}_i \in \mathrm{t}_L} \inas \Prob{\mathbf{x}_i \in \mathrm{t}_L}$ and that $N_{\mathrm{t}} / N = (1/N)\sum_{i=1}^N\ind{\mathbf{x}_i \in \mathrm{t}} \inas \Prob{\mathbf{x}_i \in \mathrm{t}}$ by the strong law of large numbers. Then, 
\begin{equation}
\widehat P_{\mathrm{t}_L} =  \frac{N_{\mathrm{t}_L} / N}{N_{\mathrm{t}} / N} \inas \Prob{\mathbf{x}_i \in \mathrm{t}_L} / \Prob{\mathbf{x}_i \in \mathrm{t}} = P_{\mathrm{t}_L}.
\label{eq:PtL}
\end{equation}
Moreover, 
\begin{equation}
\bar{y}_{\mathrm{t}_L} =\frac{N}{N_{\mathrm{\mathrm{t}_L}}}\frac{1}{N} \sum_{i = 1}^N y_i\ind{\mathbf{x}_i \in \mathrm{\mathrm{t}_L}} \inas \frac{1}{\Prob{\mathbf{x}_i \in \mathrm{t}_L}} \mathbb{E}\{y_i\ind{\mathbf{x}_i \in \mathrm{\mathrm{t}_L}}  \} =  \mathbb{E}(y \mid \vecX \in \mathrm{\mathrm{t}_L} )
\label{eq:ytL}
\end{equation}
A symmetrical argument gives
\begin{equation}
\widehat P_{\mathrm{t}_R} \inas P_{\mathrm{t}_R}.
\label{eq:PtR}
\end{equation}
and
\begin{equation}
\bar{y}_{\mathrm{t}_R} \inas \mathbb{E}(y \mid \vecX \in \mathrm{\mathrm{t}_R} ).
\label{eq:ytR}
\end{equation}
Combining (\ref{eq:PtL}), (\ref{eq:ytL}), (\ref{eq:PtR}), and (\ref{eq:ytR}), we obtain
\begin{equation}
\begin{aligned}
    \widehat{\mathcal{CS}}(j, s, \mathrm{t}) &= \widehat{P}_{\mathrm{t}_L}^2\widehat{P}_{\mathrm{t}_R}^2(\bar{y}_{t_L} - \bar{y}_{t_R})^2 \\ &\inas P_{\mathrm{t}_L}^2P_{\mathrm{t}_R}^2\left\{\mathbb{E}(y\mid \vecX \in \mathrm{t}_L) - \mathbb{E}(y\mid \vecX \in \mathrm{t}_R)\right\}^2  = {\mathcal{CS}}(j, s, \mathrm{t}),
    \end{aligned}
    \label{eq:convergence_split_criterion}
\end{equation}
pointwise at each $s \in (a_j, b_j)$.
Furthermore, using the condition that $\mathbb{E}(y^2) < \infty$ and that the function class $\{\ind{x_j \leq s} : a_j < s < b_j\}$ is a VC-subgraph class with VC dimension 1, it can be shown that the convergence in (\ref{eq:convergence_split_criterion}) is uniform on compact sets $\in [a^\prime, b^\prime] \subset (a_j, b_j)$ by the uniform law of large number. Hence, 
       \[
        \hat s_j \inp s_j^* = 
        \underset{a_j < s \leq b_j}{\operatorname{argmax}}{\ \mathcal{CS}}(j, s, \mathrm{t})
   \]
by Theorem 2.7 of \cite{kim1990cube} since for each $j$, $s_j^*$ is the unique global maximum of ${\mathcal{CS}}(j, s, \mathrm{t})$ and $\hat{s}_j$ is bounded.
\end{proof}

\subsection{Proof of Lemma 1}
\begin{proof}
Define an empirical measure $\Pi(j, s)$ on the covariates $x_j$ and split points $s$, having density
\begin{equation}
    \frac{d \Pi(j, s)}{d(j, s)}:=\frac{\left|D g_j(s)\right|}{\sum_{j^{\prime}=1}^p \int\left|D g_{j^{\prime}}\left(s^{\prime}\right)\right| d s^{\prime}},
    \label{eq:density}
\end{equation}
where $D g_j(\cdot)$ denotes the divided difference of $g_j(\cdot)$ for the successive ordered data points along the $j^{\textbf{th}}$ direction within t. Specifically, if $ x_{(1)j} \leq x_{(2)j} \leq \cdots \leq x_{(N_t)j}$ denotes the ordered data points along the $j^{\textbf{th}}$ direction within t, then 
\[
    D g_j(s):=\left\{\begin{array}{ll}
\frac{g_j\left(x_{(i+1) j}\right)-g_j\left(x_{(i) j}\right)}{x_{(i+1) j}-x_{(i) j}} & \text { if } x_{(i) j} \leq s<x_{(i+1) j} \\
0 & \text { if } s=x_{(i) j}=x_{(i+1) j}, s<x_{(1) j}, \text { or } s>x_{\left(N_{\mathrm{t}}\right) j}
\end{array} .\right.
\]

By the definition of  $\widehat{\mathcal{CS}}(\mathrm{t})=\max _{(j, s)} \widehat{\mathcal{CS}}(j, s, \mathrm{t})$,
\begin{equation}
\widehat{\mathcal{CS}}(\mathrm{t}) \geq \int \widehat{\mathcal{CS}}(j, s, \mathrm{t}) d \Pi(j, s)=\int\left|\left\langle y-\bar{y}_{\mathrm{t}}, \Psi_{\mathrm{t}}\right\rangle_{\mathrm{t}}\right|^2 d \Pi(j, s),
\label{eq:cs_lower_bound}
\end{equation}
where the last equality follows from Proposition \ref{prop:inner_product}. 

By the Jensen inequality,
\begin{equation}
    \int\left|\left\langle y-\bar{y}_{\mathrm{t}}, \Psi_{\mathrm{t}}\right\rangle_{\mathrm{t}}\right|^2 d \Pi(j, s) \geq \left(\int\left|\left\langle y-\bar{y}_{\mathrm{t}}, \Psi_{\mathrm{t}}\right\rangle_{\mathrm{t}}\right| d \Pi(j, s)\right)^2.
    \label{eq:inner_Jensen}
\end{equation}
It follows from (\ref{eq:stump_func}) that 

\[
\Phi_{\mathrm{t}}(\vecX) \ind{\vecX \in \mathrm{t}}=\ind{\vecX \in \mathrm{t}_L} \widehat{P}_{\mathrm{t}_R}-\ind{\vecX \in \mathrm{t}_R} \widehat{P}_{\mathrm{t}_L} = -\left(\ind{x_j>s}-\widehat{P}_{\mathrm{t}_R}\right) \ind{\vecX \in \mathrm{t}},
\]
and that $\left\langle y-\bar{y}_{\mathrm{t}}, \Phi_{\mathrm{t}}\right\rangle_{\mathrm{t}}=-\left\langle y-\bar{y}_{\mathrm{t}}, \ind{x_j>s}\right\rangle_{\mathrm{t}}$.
Thus, we have 
\begin{equation}
\int\left|\left\langle y-\bar{y}_{\mathrm{t}}, \Psi_{\mathrm{t}}\right\rangle_{\mathrm{t}}\right| d \Pi(j, s)=\frac{\sum_{j=1}^p \int\left|D g_j(s)\right|\left|\left\langle y-\bar{y}_{\mathrm{t}}, \ind{x_j>s}\right\rangle_{\mathrm{t}}\right| d s}{\sum_{j^{\prime}=1}^p \int\left|D g_{j^{\prime}}\left(s^{\prime}\right)\right| d s^{\prime}} .
\label{eq:main_part}
\end{equation}
For the denominator of (\ref{eq:main_part}), 
\begin{equation}
    \begin{aligned}
\sum_{j^{\prime}=1}^p\int\left|D g_{j^{\prime}}\left(s^{\prime}\right)\right|  d s^{\prime} & =\sum_{j^{\prime}=1}^p\sum_{i=1}^{N_{\mathrm{t}}-1} \int_{x_{(i) j^{\prime}}}^{x_{(i+1) j^{\prime}}}\left|D g_{j^{\prime}}\left(s^{\prime}\right)\right| d s^{\prime} \\
& =\sum_{j^{\prime}=1}^p\sum_{i=1}^{N_{\mathrm{t}}-1}\left|g_{j^{\prime}}\left(x_{(i+1) j^{\prime}}\right)-g_{j^{\prime}}\left(x_{(i) j^{\prime}}\right)\right|\\
& \leq \sum_{j^{\prime}=1}^p\operatorname{TV}\left(g_{j^{\prime}}\right) = \|g\|_{\mathrm{TV}}. 
\end{aligned}
\label{eq:denominator_of_main}
\end{equation}

For the numerator of (\ref{eq:main_part}), 
\begin{equation}
\begin{aligned}
\sum_{j=1}^p \int\left|D g_j(s)\right|\left|\left\langle y-\bar{y}_{\mathrm{t}}, \ind{x_j>s}\right\rangle_{\mathrm{t}}\right| d s &\geq\left|\sum_{j=1}^p \int D g_j(s)\left\langle y-\bar{y}_{\mathrm{t}}, \ind{x_j>s}\right\rangle_{\mathrm{t}} d s\right| \\
& =\left|\left\langle y-\bar{y}_{\mathrm{t}}, \sum_{j=1}^p \int D g_j(s) \mathbf{1}\left(x_j>s\right) d s\right\rangle_{\mathrm{t}}\right| \\
& =\left|\left\langle y-\bar{y}_{\mathrm{t}}, \sum_{j=1}^p g_j\right\rangle_{\mathrm{t}}\right| \\
& =\left|\left\langle y-\bar{y}_{\mathrm{t}}, g\right\rangle_{\mathrm{t}}\right|.
\end{aligned}
\label{eq:numerator_of_main}
\end{equation}
Moreover, 
\begin{equation}
\begin{aligned}
\left\langle y-\bar{y}_{\mathrm{t}}, g\right\rangle_{\mathrm{t}} &=\left\langle y-\bar{y}_{\mathrm{t}}, y\right\rangle_{\mathrm{t}}-\left\langle y-\bar{y}_{\mathrm{t}}, y - g\right\rangle_{\mathrm{t}} \\
&\geq\left\|y-\bar{y}_{\mathrm{t}}\right\|_{\mathrm{t}}^2-\left\|y-\bar{y}_{\mathrm{t}}\right\|_{\mathrm{t}}\|y-g\|_{\mathrm{t}} \\
& \geq \frac{\left\|y-\bar{y}_{\mathrm{t}}\right\|_{\mathrm{t}}^2 - \|y-g\|_{\mathrm{t}}^2}{2}.
\end{aligned}
\label{eq:acc1}
\end{equation}
where the second to last inequality follows from the Cauchy-Schwarz inequality and the last is due to AM–GM inequality.

Combining (\ref{eq:cs_lower_bound}), (\ref{eq:inner_Jensen}), (\ref{eq:main_part}), (\ref{eq:denominator_of_main}), (\ref{eq:numerator_of_main}), (\ref{eq:acc1}), and using the assumption that $\left\|y-\bar{y}_{\mathrm{t}}\right\|_{\mathrm{t}}^2 - \|y-g\|_{\mathrm{t}}^2 \geq 0$, we obtain
\[
    \widehat{\mathcal{CS}}(\mathrm{t}) \geq \frac{\left(\left\|y-\bar{y}_{\mathrm{t}}\right\|_{\mathrm{t}}^2-\|y-g\|_{\mathrm{t}}^2\right)^2}{4\|g\|_{\mathrm{TV}}^2}=\frac{\left\{\widehat{\mathcal{R}}_{\mathrm{t}}\left(\widehat T^\text{Cov}_{K-1}\right)-\widehat{\mathcal{R}}_{\mathrm{t}}(g)\right\}^2}{4\|g\|_{\mathrm{TV}}^2}.
\]
\end{proof}

        \subsection{Proof of Theorem 3}
\begin{proof}
    Notice that for any covariate $x_j$ split point $s$ at any node t, the impurity gain
    \[
        \mathcal{I} \mathcal{G}(j, s, \mathrm{t})  = \widehat{P}_{\mathrm{t}_L}\widehat{P}_{\mathrm{t}_R}(\bar{y}_{\mathrm{t}_L} - \bar{y}_{\mathrm{t}_R})^2 = \frac{1}{\widehat{P}_{\mathrm{t}_L}\widehat{P}_{\mathrm{t}_R}} \widehat{\mathcal{CS}}(j, s, \mathrm{t}) \geq 4\widehat{\mathcal{CS}}(j, s, \mathrm{t}),
    \]
    where the first equality follows from \cite[Section 9.3]{breiman1984}.
    Hence, we have
    \begin{equation}
    \mathcal{I} \mathcal{G}(\mathrm{t}) =\max _{(j, s)} \mathcal{I} \mathcal{G}(j, s, \mathrm{t}) \geq 4\max _{(j, s)}\widehat{\mathcal{CS}}(j, s, \mathrm{t}) = 4\widehat{\mathcal{CS}}(\mathrm{t}) \geq \frac{\left\{\widehat{\mathcal{R}}_{\mathrm{t}}\left(\widehat T^\text{Cov}_{K-1}\right)-\widehat{\mathcal{R}}_{\mathrm{t}}(g)\right\}^2}{\|g\|_{\mathrm{TV}}^2},
    \label{eq:impurityh_gain}
\end{equation}
where the last inequality comes from Lemma 1.
The impurity gain inequality in Equation~(\ref{eq:impurityh_gain}) for CovRT is identical to the one established for CART in Lemma 4.1 of \cite{klusowski2024}. However, since CovRT does not aim to maximize impurity gain at each split, the following equation, which holds for CART,
\[
\widehat{\mathcal{R}}\left(\hat{g}\left(\widehat T^\text{CART}_{K}\right)\right)=\widehat{\mathcal{R}}\left(\hat{g}\left(\widehat T^\text{CART}_{K-1}\right)\right)-\sum_{\mathrm{t} \in T_{K-1}} \frac{N_{\mathrm{t}}}{N} \mathcal{I} \mathcal{G}(\mathrm{t})
\]
deos not necessarily hold for CovRT. Consequently, Lemma D.1 in the Appendix of \cite{klusowski2024} is not directly applicable for our proof.

Instead, we first study the empirical risks for CovRT with depths $K-1$ and $K$. For the tree with depth $K - 1$, we have
\begin{equation}
    \begin{aligned}\widehat{\mathcal{R}}\left(\hat{g}\left(\widehat T^\text{Cov}_{K-1}\right)\right) &= \frac{1}{N} \sum_{i=1}^N \left\{y_i - \hat{g}\left(\widehat T^\text{Cov}_{K-1}\right)\left(\mathbf{x}_i\right)\right\}^2 \\
    &= \sum_{\mathrm{t} \in \widehat T^\text{Cov}_{K-1}} \frac{N_{\mathrm{t}}}{N}\left\{\frac{1}{N_{\mathrm{t}}}\sum_{\vecX_i \in \mathrm{t}}\left(y_i - \bar{y}_{\mathrm{t}}\right)^2\right\},
    \end{aligned}
    \label{eq:empirical_risk_K-1}
\end{equation}
where the sum is over all terminal nodes of $\widehat T^\text{Cov}_{K-1}$. For the tree with depth $K$, we have
\begin{equation}
    \begin{aligned}\widehat{\mathcal{R}}\left(\hat{g}\left(\widehat T^\text{Cov}_{K}\right)\right) &= \frac{1}{N} \sum_{i=1}^N \left\{y_i - \hat{g}\left(\widehat T^\text{Cov}_{K}\right)\left(\mathbf{x}_i\right)\right\}^2 \\
    &= \sum_{\mathrm{t} \in \widehat T^\text{Cov}_{K - 1}}\frac{N_{\mathrm{t}}}{N}\left\{\frac{1}{N_{\mathrm{t}}}\sum_{\vecX_i \in \mathrm{t}_L }\left(y_i - \bar{y}_{\mathrm{t}_L}\right)^2 + \frac{1}{N_{\mathrm{t}}}\sum_{\vecX_i \in \mathrm{t}_R}\left(y_i - \bar{y}_{\mathrm{t}_R}\right)^2\right\} ,
    \end{aligned}
    \label{eq:empirical_risk_K}
\end{equation}
where the sum is over all terminal nodes of $\widehat T^\text{Cov}_{K-1}$ and $\mathrm{t}_L$ and $\mathrm{t}_R$ are the left and right daughter nodes generated from the terminal node $\mathrm{t} \in \widehat T^\text{Cov}_{K - 1}$. Subtracting (\ref{eq:empirical_risk_K}) from (\ref{eq:empirical_risk_K-1}) gives
\begin{equation}
    \begin{aligned}
    \widehat{\mathcal{R}}\left(\hat{g}\left(\widehat T^\text{Cov}_{K-1}\right)\right) - \widehat{\mathcal{R}}\left(\hat{g}\left(\widehat T^\text{Cov}_{K}\right)\right) &= \sum_{\mathrm{t} \in \widehat T^\text{Cov}_{K - 1}}\frac{N_{\mathrm{t}}}{N} \frac{1}{\widehat{P}_{\mathrm{t}_L}\widehat{P}_{\mathrm{t}_R}} \widehat{\mathcal{CS}}(\mathrm{t}) \\
    &\geq 4\sum_{\mathrm{t} \in \widehat T^\text{Cov}_{K - 1}}\frac{N_{\mathrm{t}}}{N}  \widehat{\mathcal{CS}}(\mathrm{t}),
    \end{aligned}
    \label{eq:empirical_risk_diff}
\end{equation}
where the first equality follows from that 
\[
 \widehat{\mathcal{CS}}(j, s, \mathrm{t}) = \widehat{P}_{\mathrm{t}_L}\widehat{P}_{\mathrm{t}_R} \left[ \frac{1}{N_{\mathrm{t}}} \sum_{\mathbf{x}_i \in \mathrm{t}} \left(y_i - \bar{y}_{\mathrm{t}}\right)^2 - \frac{1}{N_{\mathrm{t}}}\left\{\sum_{\mathbf{x}_i \in \mathrm{t}_L} \left(y_i - \bar{y}_{\mathrm{t}_L}\right)^2 + \sum_{\mathbf{x}_i \in \mathrm{t}_R} \left(y_i - \bar{y}_{\mathrm{t}_R}\right)^2\right\}\right],
\]
and that CovRT maximizes $\widehat{\mathcal{CS}}(j, s, \mathrm{t})$ at each split.

We now begin showing for any candidate model $g \in \mathcal{G}^1$ and any depth $K \geq 1$, 
\begin{equation}
\widehat{\mathcal{R}}\left(\hat{g}\left(\widehat T^\text{Cov}_{K}\right)\right) \leq \widehat{\mathcal{R}}(g)+\frac{\|g\|_{\mathrm{TV}}^2}{K+3}.
\label{eq:excess_risk_inequality}
\end{equation}
To this end, we define \[\mathcal{E}_K:=\widehat{\mathcal{R}}\left(\hat{g}\left(\widehat T^\text{Cov}_{K}\right)\right)-\widehat{\mathcal{R}}(g), \quad \mathcal{E}_K(\mathrm{t}):=\widehat{\mathcal{R}}_{\mathrm{t}}\left(\hat{g}\left(\widehat T^\text{Cov}_{K}\right)\right)-\widehat{\mathcal{R}}_{\mathrm{t}}(g)
    \]
as the global and within-node excess empirical risks, respectively. It follows from (\ref{eq:empirical_risk_diff}) that the empirical risk decreases as tree depth increases, i.e., $\mathcal{E}_1 \geq \mathcal{E}_2 \geq \cdots \geq \mathcal{E}_K$. When $\mathcal{E}_{K-1} < 0$, it immediately follows that $\mathcal{E}_{K} \leq \mathcal{E}_{K-1} < 0$, and thus inequality (\ref{eq:excess_risk_inequality}) is automatically satisfied.

On the other hand, when $\mathcal{E}_{K-1} \geq 0$, (\ref{eq:empirical_risk_diff}) implies that
\begin{equation}
    \begin{aligned}
    \mathcal{E}_{K-1} - \mathcal{E}_{K}
    &\geq 4\sum_{\mathrm{t} \in \widehat T^\text{Cov}_{K - 1}}\frac{N_{\mathrm{t}}}{N}  \widehat{\mathcal{CS}}(\mathrm{t}) \\
    &= 4\sum_{\mathrm{t} \in \widehat T^\text{Cov}_{K - 1},  \mathcal{E}_{K-1}(t) > 0}\frac{N_{\mathrm{t}}}{N}  \widehat{\mathcal{CS}}(\mathrm{t}) +  4\sum_{\mathrm{t} \in \widehat T^\text{Cov}_{K - 1},  \mathcal{E}_{K-1}(t) \leq 0}\frac{N_{\mathrm{t}}}{N}  \widehat{\mathcal{CS}}(\mathrm{t}) \\
    &\geq 4\sum_{\mathrm{t} \in \widehat T^\text{Cov}_{K - 1},  \mathcal{E}_{K-1}(t) > 0}\frac{N_{\mathrm{t}}}{N}  \widehat{\mathcal{CS}}(\mathrm{t}) \\
    &\geq \frac{1}{\|g\|_{\mathrm{TV}}^2}\sum_{\mathrm{t} \in \widehat T^\text{Cov}_{K - 1},  \mathcal{E}_{K-1}(t) > 0}\frac{N_{\mathrm{t}}}{N} { \mathcal{E}^2_{K-1}(t)},
    \end{aligned}
    \label{eq:empirical_risk_diff2}
\end{equation}
where the last inequality follows from Lemma 1.

Then, by Jensen's inequality, 
\begin{equation}
\begin{aligned}
\sum_{\mathrm{t} \in \widehat T^\text{Cov}_{K - 1}, \mathcal{E}_{K-1}(\mathrm{t})>0} \frac{N_{\mathrm{t}}}{N} \mathcal{E}_{K-1}^2(\mathrm{t}) & \geq\left(\sum_{\mathrm{t} \in \widehat T^\text{Cov}_{K - 1}, \mathcal{E}_{K-1}(\mathrm{t})>0} \frac{N_{\mathrm{t}}}{N} \mathcal{E}_{K-1}(\mathrm{t})\right)^2 \\
& \geq\left(\sum_{\mathrm{t} \in \widehat T^\text{Cov}_{K - 1}} \frac{N_{\mathrm{t}}}{N} \mathcal{E}_{K-1}(\mathrm{t})\right)^2 \\
& =\mathcal{E}_{K-1}^2,
\end{aligned}
\label{eq:risk_square}
\end{equation}
where the second inequality holds because $\mathcal{E}_{K-1} \geq 0$. 
Substituting (\ref{eq:risk_square}) into (\ref{eq:empirical_risk_diff2}), we get
\begin{equation}
\mathcal{E}_K \leq \mathcal{E}_{K-1}\left(1-\frac{\mathcal{E}_{K-1}}{\|g\|_{\mathrm{TV}}^2}\right), \quad K \geq 1.
\label{eq:induction_start}
\end{equation}

To finish the proof, we use mathematical induction to show that for all $K \geq 1$
\begin{equation}
\mathcal{E}_K \leq \frac{\|g\|_{\mathrm{TV}}^2}{K + 3}.
\label{eq:induction_target}
\end{equation}
For the base case when $K = 1$, by (\ref{eq:induction_start}) we have
\[
\mathcal{E}_K \leq \mathcal{E}_{K-1}\left(1-\frac{\mathcal{E}_{K-1}}{\|g\|_{\mathrm{TV}}^2}\right) \leq \frac{\|g\|_{\mathrm{TV}}^2}{4} = \frac{\|g\|_{\mathrm{TV}}^2 }{K + 3} .
\]
Assume that (\ref{eq:induction_target}) holds for $K = k \geq 1$, for $K = k + 1$, if 
$
    \mathcal{E}_{k} \leq \frac{\|g\|_{\mathrm{TV}}^2 }{k + 4},
$
then by the monotonicity of $\mathcal{E}_{K}$,
\[
\mathcal{E}_{K} \leq \mathcal{E}_{k} \leq  \frac{\|g\|_{\mathrm{TV}}^2 }{k + 4} = \frac{\|g\|_{\mathrm{TV}}^2 }{K + 3},
\]
which completes the inductive step. Otherwise, if
$
    \mathcal{E}_{k} > \frac{\|g\|_{\mathrm{TV}}^2 }{k + 4},
$
then by (\ref{eq:induction_start}),
\[
    \mathcal{E}_K \leq \mathcal{E}_{k}\left(1-\frac{\mathcal{E}_{k}}{\|g\|_{\mathrm{TV}}^2}\right) \leq \frac{\|g\|_{\mathrm{TV}}^2 }{k + 3} \left(1 - \frac{1}{k + 4}\right) = \frac{\|g\|_{\mathrm{TV}}^2 }{k + 4} = \frac{\|g\|_{\mathrm{TV}}^2 }{K + 3},
\]
where the second inequality follows from the inductive assumption and the inductive step is completed. As a result, by mathematical induction, (\ref{eq:induction_target}) holds for all $K \geq 1$, which implies that
\[
\widehat{\mathcal{R}}\left(\hat{g}\left(\widehat T^\text{Cov}_{K}\right)\right)\leq \widehat{\mathcal{R}}(g)  + \frac{\|g\|_{\mathrm{TV}}^2}{K + 3},
\]
for any candidate model $g \in \mathcal{G}^1$ and any depth $K \geq 1$. Taking the infimum of $g$ over $\mathcal{G}^1$ finishes the proof.
\end{proof}

\subsection{Proof of Theorem 4}
The proof of Theorem 4 follows by combining the result of Theorem 3 with the general theory of data-dependent partitioning estimates \citep[Chapter~13]{gyorfi2002}, as established by Theorem 4.3 in \cite{klusowski2024}.

Specifically, we first assume that the response data is bounded, i.e., $\left|y_i\right| \leq U, i=1,2, \ldots, N$, for some $U \geq \left\|g^*\right\|_{\infty}$. We decompose the $L_2$ error of $\widehat T^\text{Cov}_{K}$ as   $\left\|\hat{g}\left(\widehat T^\text{Cov}_{K}\right)-g^*\right\|^2=E_1 + E_2$, where 
\begin{equation}
E_1:=\left\|\hat{g}\left(\widehat T^\text{Cov}_{K}\right)-g^*\right\|^2-2\left(\left\|y-\hat{g}\left(\widehat T^\text{Cov}_{K}\right)\right\|_{\mathcal{D}}^2-\left\|y-g^*\right\|_{\mathcal{D}}^2\right)-\alpha-\beta,
\end{equation}
and 
\begin{equation}
E_2:=2\left(\left\|y-\hat{g}\left(\widehat T^\text{Cov}_{K}\right)\right\|_{\mathcal{D}}^2-\left\|y-g^*\right\|_{\mathcal{D}}^2\right)+\alpha+\beta,
\end{equation}
and $\alpha, \beta > 0$ to be chosen later. Then, by Theorem 3, for any candidate model $g \in \mathcal{G}_1$, we have
\begin{equation}
E_2 \leq 2\left(\|y-g\|_{\mathcal{D}}^2-\left\|y-g^*\right\|_{\mathcal{D}}^2\right)+\frac{2\|g\|_{\mathrm{TV}}^2}{K+3}+\alpha+\beta.
\end{equation}
Taking expectations on both sides gives 
\begin{equation}
\begin{aligned}
\mathbb{E}_{\mathcal{D}}\left(E_2\right) & \leq 2 \mathbb{E}_{\mathcal{D}}\left(\|y-g\|_{\mathcal{D}}^2-\left\|y-g^*\right\|_{\mathcal{D}}^2\right)+\frac{2\|g\|_{\mathrm{TV}}^2}{K+3}+\alpha+\beta \\
& =2\left\|g-g^*\right\|^2+\frac{2\|g\|_{\mathrm{TV}}^2}{K+3}+\alpha+\beta,
\end{aligned}
\label{eq:expectation_E2}
\end{equation}
where the last equality follows from $\mathbb{E}_{\mathcal{D}}\left(\|y-g\|_{\mathcal{D}}^2-\left\|y-g^*\right\|_{\mathcal{D}}^2\right) = \left\|g-g^*\right\|^2$ from the law of iterated expectations.

To bound $E_1$, let 
\begin{equation}
\Pi_N:=\left\{\mathcal{P}\left(\left\{\left(\mathbf{x}_1, y_1\right),\left(\mathbf{x}_2, y_2\right), \ldots,\left(\mathbf{x}_N, y_N\right)\right\}\right):\left(\mathbf{x}_i, y_i\right) \in \mathbb{R}^p \times \mathbb{R}\right\}
\end{equation}
 be the family of all achievable partitions $\mathcal{P}$ by growing a depth $K$ tree on $N$ points. Define 
\[
M\left(\Pi_N\right):=\max \left\{\# \mathcal{P}: \mathcal{P} \in \Pi_N\right\}
\]
 to be the maximum number of terminal nodes among all partitions in $\Pi_N$. For a set $\mathbf{z}^N=\left\{\mathbf{z}_1, \mathbf{z}_2, \ldots, \mathbf{z}_N\right\} \subset \mathbb{R}^p$, let $\Delta\left(\mathbf{z}^N, \Pi_N\right)$ be the number of distinct partitions of $\mathbf{z}^N$ induced by elements of $\Pi_N$. In other words, $\Delta\left(\mathbf{z}^N, \Pi_N\right)$ is the number of different partitions $\left\{\mathbf{z}^N \cap A: A \in \mathcal{P}\right\}$, for $\mathcal{P} \in \Pi_N$. The partitioning number $\Delta_N\left(\Pi_N\right)$ is defined by 
 \[
 \Delta_N\left(\Pi_N\right):=\max \left\{\Delta\left(\mathbf{z}^N, \Pi_N\right): \mathbf{z}_1, \mathbf{z}_2, \ldots, \mathbf{z}_N \in \mathbb{R}^p\right\},
 \]
 which is the maximum number of different partitions of any set of $N$ points that can be induced by members of $\Pi_N$. Let $\mathcal{G}_N$ denote the collection of all piecewise constant functions, bounded by $U$, defined on partitions $\mathcal{P} \in \Pi_N$.

It follows from Theorem 11.4 of \cite{gyorfi2002}, by setting $\epsilon = 1/2$ in their notation, that
\begin{equation}
\begin{aligned}
\Prob{E_1 \geq 0} & \leq \ProbB{\exists g(\cdot) \in \mathcal{G}_N:\left\|g-g^*\right\|^2 \geq 2\left(\|y-g\|_{\mathcal{D}}^2-\left\|y-g^*\right\|_{\mathcal{D}}^2\right)+\alpha+\beta}\\
& \leq 14 \sup _{\mathbf{x}^N} \mathcal{N}\left(\frac{\beta}{40 U}, \mathcal{G}_N, \mathscr{L}_1\left(\mathbb{P}_{\mathbf{x}^N}\right)\right) \exp \left(-\frac{\alpha N}{2568 U^4}\right),
\end{aligned}
\label{eq:E1_tail_bound}
\end{equation}
where $\mathbf{x}^N=\left\{\mathbf{x}_1, \mathbf{x}_2, \ldots, \mathbf{x}_N\right\} \subset \mathbb{R}^p$ and $\mathcal{N}\left(r, \mathcal{G}_N, \mathscr{L}_1\left(\mathbb{P}_{\mathbf{x}^N}\right)\right)$ denotes the covering number for $\mathcal{G}_N$ by $\mathscr{L}_1$ balls of radius $r > 0$ with respect to the empirical discrete measure $\mathbb{P}_{\mathbf{x}^N}$ on $\mathbf{x}^N$.

Next, we bound the covering number using Lemma 13.1 and Theorem 9.4 of \cite{gyorfi2002}, following the argument in their proof of Theorem 13.1 with $\epsilon = ({40}/{32})\beta$. Specifically, we have
\begin{equation}
\mathcal{N}\left(\frac{\beta}{40 U}, \mathcal{G}_N, \mathscr{L}_1\left(\mathbb{P}_{\mathbf{x}^N}\right)\right) \leq \Delta_N\left(\Pi_N\right)\left(\frac{40}{32} \cdot \frac{333 e U^2}{\beta}\right)^{2 M\left(\Pi_N\right)} \leq \Delta_N\left(\Pi_N\right)\left(\frac{417 e U^2}{\beta}\right)^{2^{K+1}},
\end{equation}
where the last inequality uses that any binary partition tree of depth $K$ has at most $2^K$ terminal nodes, i.e., $M(\Pi_N) \leq 2^K$.

\cite{klusowski2024} further shows that the partitioning number $\Delta_N\left(\Pi_N\right)$ satisfies 
\[
\Delta_N\left(\Pi_N\right) \leq((N-1) p)^{2^K-1} \leq(N p)^{2^K}.
\]
Thus, we can further bound the covering number by
\begin{equation}
\mathcal{N}\left(\frac{\beta}{40 U}, \mathcal{G}_N, \mathscr{L}_1\left(\mathbb{P}_{\mathbf{x}^N}\right)\right) \leq(N p)^{2^K}\left(\frac{417 e U^2}{\beta}\right)^{2^{K+1}}.
\label{eq:cover_bound}
\end{equation}
Applying  (\ref{eq:cover_bound}) to (\ref{eq:E1_tail_bound}) yields
\[
\Prob{E_1 \geq 0} \leq 14(N p)^{2^K}\left(\frac{417 e U^2}{\beta}\right)^{2^{K+1}} \exp \left(-\frac{\alpha N}{2568 U^4}\right).
\]
We can choose $\alpha = O(\frac{U^4 2^K \log (N p)}{N})$ and $\beta = O(\frac{U^2}{N})$ such that $\Prob{E_1 \geq 0} \leq C_1^{\prime} / N$ for some universal constant $C_1^{\prime} > 0$. Moreover, since the response variable is bounded by $U$ by assumption, we have $E_1 \leq\left\|\hat{g}\left(\widehat T^\text{Cov}_K\right)-g^*\right\|^2+2\left\|y-g^*\right\|_{\mathcal{D}}^2 \leq 12 U^2$. It then follows that 
\begin{equation}
\mathbb{E}_{\mathcal{D}}\left(E_1\right) \leq 12 U^2 \cdot \Prob{E_1 \geq 0} \leq \frac{12 C_1^{\prime} U^2}{N}.
\label{eq:expectation_E1}
\end{equation}
Combining (\ref{eq:expectation_E2}) and (\ref{eq:expectation_E1}), and substituting the chosen values of $\alpha$ and $\beta$, we obtain
\begin{equation}
\begin{aligned}
\mathbb{E}_{\mathcal{D}}\left(\left\|\hat{g}\left(\widehat T^\text{Cov}_K\right)-g^*\right\|^2\right) & =\mathbb{E}_{\mathcal{D}}\left(E_1\right)+\mathbb{E}_{\mathcal{D}}\left(E_2\right) \\
& \leq 2\left\|g-g^*\right\|^2+\frac{2\|g\|_{\mathrm{TV}}^2}{K+3}+C_1^{\prime \prime}\left(\frac{U^4 2^K \log (N p)}{N}+\frac{U^2}{N}\right),
\end{aligned}
\label{eq:risk_bound_bounded}
\end{equation}
where $C_1^{\prime \prime} > 0$ is a universal constant.

When the response data is unbounded, let $E=\bigcap_i\left\{\left|y_i\right| \leq U\right\}$, where $U \geq\left\|g^*\right\|_{\infty}$. Then, 
\begin{equation}
\begin{aligned}
\mathbb{E}_{\mathcal{D}}\left(\left\|\hat{g}\left(\widehat T^\text{Cov}_K\right)-g^*\right\|^2\right)= & \mathbb{E}_{\mathcal{D}}\left(\left\|\hat{g}\left(\widehat T^\text{Cov}_K\right)-g^*\right\|^2 \ind{E}\right)+\mathbb{E}_{\mathcal{D}}\left(\left\|\hat{g}\left(\widehat T^\text{Cov}_K\right)-g^*\right\|^2 \ind{E^c}\right) \\
\leq & 2\left\|g-g^*\right\|^2+\frac{2\|g\|_{\mathrm{TV}}^2}{K+3}+C_1^{\prime \prime}\left(\frac{U^4 2^K \log (N p)}{N}+\frac{U^2}{N}\right) \\
& +\mathbb{E}_{\mathcal{D}}\left(\left\|\hat{g}\left(\widehat T^\text{Cov}_K\right)-g^*\right\|^2 \ind{E^c}\right),
\end{aligned}
\label{eq:error_overall}
\end{equation}
 where $E^c$ denotes the complement of the event $E$ and the last inequality follows from (\ref{eq:risk_bound_bounded}). Furthermore, by the subadditivity of probability and the sub-Gaussian noise condition, we have
\begin{equation}
    \Prob{E^c} \leq N\Prob{\left|y\right| > U} \leq N\Prob{\left|y\right| > U - \left\|g^*\right\|_{\infty}} \leq 2 N \exp \left(-\frac{\left(U-\left\|g^*\right\|_{\infty}\right)^2}{2 \sigma^2}\right).
\end{equation}
Then, by the Cauchy-Schwarz inequality
\begin{equation}
\begin{aligned}
\mathbb{E}_{\mathcal{D}}\left(\left\|\hat{g}\left(\widehat T^\text{Cov}_K\right)-g^*\right\|^2 \ind{E^c}\right) 
& \leq \sqrt{\mathbb{E}_{\mathcal{D}}\left(\left\|\hat{g}\left(\widehat T^\text{Cov}_K\right)-g^*\right\|^4\right) \Prob{E^c}} \\
& \leq \sqrt{2 N \cdot \mathbb{E}_{\mathcal{D}}\left(\left\|\hat{g}\left(\widehat T^\text{Cov}_K\right)-g^*\right\|^4\right)} \exp \left(-\frac{\left(U-\left\|g^*\right\|_{\infty}\right)^2}{4 \sigma^2}\right) \\
& \leq \sqrt{2 N \cdot \left[\left\{\mathbb{E}_{\mathcal{D}}\left(\left\|\hat{g}\left(\widehat T^\text{Cov}_K\right)\right\|^4\right)\right\}^{1/4} + \left\|g^*\right\|_{\infty} \right]^4} \exp \left(-\frac{\left(U-\left\|g^*\right\|_{\infty}\right)^2}{4 \sigma^2}\right) \\
& \leq \sqrt{16 N\left\{\mathbb{E}_{\mathcal{D}}\left(\left\|\hat{g}\left(\widehat T^\text{Cov}_K\right)\right\|^4\right)+\left\|g^*\right\|_{\infty}^4\right\}} \exp \left(-\frac{\left(U-\left\|g^*\right\|_{\infty}\right)^2}{4 \sigma^2}\right),
\end{aligned}
\label{eq:error_Ec}
\end{equation}
where the second to last inequality follows from the triangular inequality with respect to $\mathscr{L}^4$ norm and the last one follows from the inequality $(a + b)^4 \leq 8(a^4 + b^4)$. Notice that $\mathbb{E}_{\mathcal{D}}\left(\left\|\hat{g}\left(\widehat T^\text{Cov}_K\right)\right\|^4\right) \leq \mathbb{E}_{\mathcal{D}}\left(\max _i\left|y_i\right|^4\right) \leq N \mathbb{E}\left(|y|^4\right) \leq 8 N\left\{\left\|g^*\right\|_{\infty}^4+\mathbb{E}\left(\varepsilon^4\right)\right\}$. Substituting $U=\left\|g^*\right\|_{\infty}+2 \sigma \sqrt{2 \log (N)}$ into (\ref{eq:error_Ec}), we obtain 
\[
\mathbb{E}_{\mathcal{D}}\left(\left\|\hat{g}\left(\widehat T^\text{Cov}_K\right)-g^*\right\|^2 \ind{E^c}\right) \leq \frac{C_1^{\prime \prime \prime}}{N},
\] for some constant $C_1^{\prime \prime \prime} > 0$ that depends only on $\left\|g^*\right\|_{\infty}$ and $\sigma^2$. Plugging this choice of $U$ and applying (\ref{eq:error_Ec}) to (\ref{eq:error_overall}), we obtain
\[
\mathbb{E}_{\mathcal{D}}\left(\left\|\hat{g}\left(\widehat T^\text{Cov}_K\right)-g^*\right\|^2\right) \leq 2\left\|g-g^*\right\|^2+\frac{2\|g\|_{\mathrm{TV}}^2}{K+3}+C_1 \frac{2^K \log ^2(N) \log (N p)}{N},
\]
for some constant $C_1 > 0$ that depends only on $\left\|g^*\right\|_{\infty}$ and $\sigma^2$. The proof is finished by taking the infimum of $g$ over $\mathcal{G}^1$.

\end{document}